%% file: main.tex
\documentclass[letterpaper, 10 pt, conference]{ieeeconf}  % Comment this line out if you need a4paper
\IEEEoverridecommandlockouts                              % This command is only
\usepackage{graphics}
\usepackage{epsfig}
\usepackage{mathptmx} 
\usepackage{mathrsfs}
\usepackage{times}
\usepackage{amsmath}

\usepackage{amsthm}
\usepackage{amssymb}
\usepackage{wasysym}
\usepackage{txfonts}
\usepackage{bbm}
\usepackage[singlelinecheck=false,font=footnotesize]{caption}
\usepackage{subcaption}
\usepackage{color}
\usepackage[dvipsnames]{xcolor}
\usepackage{indentfirst}
\usepackage{multirow}

\usepackage{color, colortbl}
\usepackage{algorithm}
\usepackage{algpseudocode}

\usepackage[
    style=ieee,
    doi=false,
    isbn=false,
    url=false,
    eprint=false,
    backend=bibtex,
    natbib=true
    ]{biblatex}

\graphicspath{{Images/}}

\bibliography{references}

\include{commands}

\usepackage{hyperref}

\title{\LARGE \bf Technical Report: Safe, Aggressive Quadrotor Flight via Reachability-based Trajectory Design}

\author{Shreyas Kousik, Patrick Holmes, Ram Vasudevan$^*$
\thanks{This work has been accepted to the 2019 ASME Dynamic Systems and Control Conference.
The authors are supported by the Office of Naval Research under award number N00014-18-1-2575, and by the National Science Foundation Award \#1751093}
\thanks{$^{*}$Mechanical Engineering, University of Michigan, Ann Arbor, MI {\tt\small <skousik,pdholmes,ramv>@umich.edu}}
}
\begin{document}

\maketitle
\thispagestyle{empty}
\pagestyle{plain}

\begin{abstract}\emph{ % for ASME
Quadrotors can provide services such as infrastructure inspection and search-and-rescue, which require operating autonomously in cluttered environments. 
Autonomy is typically achieved with receding-horizon planning, where a short plan is executed while a new one is computed, because sensors receive limited information at any time.
To ensure safety and prevent robot loss, plans must be verified as collision free despite uncertainty (e.g, tracking error).
Existing spline-based planners dilate obstacles uniformly to compensate for uncertainty, which can be conservative.
On the other hand, reachability-based planners can include trajectory-dependent uncertainty as a function of the planned trajectory.
This work applies Reachability-based Trajectory Design (RTD) to plan quadrotor trajectories that are safe despite trajectory-dependent tracking error.
This is achieved by using zonotopes in a novel way for online planning.
Simulations show aggressive flight up to 5 m/s with zero crashes in 500 cluttered, randomized environments.
}
\end{abstract}

\input{sections/01_introduction.tex}
\input{sections/02_dynamic_models.tex}
\input{sections/03_tracking_error.tex}

\input{sections/04_reachability.tex}

\input{sections/05_online_planning.tex}
\input{sections/06_results.tex}
\input{sections/07_conclusion.tex}

\renewcommand{\bibfont}{\normalfont\small}
{\renewcommand{\markboth}[2]{}
\printbibliography}

\end{document}

%% file: commands.tex
\newtheorem{defn}{Definition}

\newtheorem{lem}[defn]{Lemma}
\newtheorem{prop}[defn]{Proposition}
\newtheorem{assum}[defn]{Assumption}

\newtheorem{thm}[defn]{Theorem}

\providecommand{\R}{\ensuremath \mathbb{R}}
\providecommand{\N}{\ensuremath \mathbb{N}}
\providecommand{\T}{\ensuremath T}
\providecommand{\tfin}{t_\regtext{f}}

\renewcommand{\P}{\ensuremath \mathcal{P}}

\newcommand{\norm}[1]{\left\Vert#1\right\Vert}

\newcommand{\defemph}[1]{\emph{#1}}

\newcommand{\inv}{^{-1}}

\newcommand{\ts}[1]{\textsuperscript{#1}}

\newcommand{\kp}{\kappa}

\newcommand{\regtext}[1]{\mathrm{\textnormal{#1}}}

\newcommand{\hi}{_\regtext{hi}}

\newcommand{\FRS}{\tilde{F}}
\newcommand{\FRSexact}{F}

\newcommand{\plan}{_\regtext{plan}}
\newcommand{\peak}{_\regtext{pk}}

\newcommand{\peaki}{_{\regtext{pk},i}}
\newcommand{\peakj}{_{\regtext{pk}}^{(j)}}

\newcommand{\upi}{^{(i)}}
\newcommand{\upj}{^{(j)}}
\newcommand{\upn}{^{(n)}}
\newcommand{\upm}{^{(m)}}
\newcommand{\sense}{_\regtext{sense}}

\newcommand{\des}{_\regtext{des}}

\newcommand{\unsafe}{_\regtext{u}}

\newcommand{\vol}{\regtext{vol}}

\newcommand{\idx}{\regtext{proj}_X}

\newcommand{\idpeak}{\regtext{proj}_{K\peak}}
\newcommand{\vmax}{{v_\regtext{max}}}
\newcommand{\amax}{{a_\regtext{max}}}

\newcommand{\bt}{\beta}
\newcommand{\dl}{\delta}
\newcommand{\Dl}{\Delta}
\newcommand{\gm}{\gamma}

\newcommand{\om}{\omega}
\newcommand{\Om}{\Omega}

\definecolor{Gray}{gray}{0.9}
\newcolumntype{g}{>{\columncolor{Gray}}c}

\newcommand{\SO}{\mathsf{SO}}

\newcommand{\oneD}{_{\regtext{1D}}}
\newcommand{\oneDt}{_{\regtext{1D}}^{(t)}}
\newcommand{\oneDo}{_{\regtext{1D}}^{(0)}}

\newcommand{\gapp}{\tilde{g}}

\newcommand{\Dcover}{\mathscr{D}}
\newcommand{\nObs}{n_\mathscr{O}}
\newcommand{\obsset}{\mathscr{O}}

\newcommand{\indext}{^{(t)}}

\newcommand{\BQR}{B_{\regtext{QR}}}

\newcommand{\QR}{_{\regtext{QR}}}

\newcommand{\diag}{\regtext{diag}}
\newcommand{\boxfun}{\regtext{box}}

\newcommand{\sensorhorizon}{d\sense}

\newcommand{\epsig}{_{\epsilon, \sigma}}
\newcommand{\pkunsafe}{_\regtext{pk,u}}

\newcommand{\uptj}{^{(t, j)}}

\newcommand{\pkunsafei}{_\regtext{pk,u,i}}
\newcommand{\pkunsafeone}{_\regtext{pk,u,1}}
\newcommand{\pkunsafetwo}{_\regtext{pk,u,2}}
\newcommand{\pkunsafethree}{_\regtext{pk,u,3}}

\newcommand{\gmxv}{\gamma_{x,v}}
\newcommand{\gmxa}{\gamma_{x,a}}
\newcommand{\gmxpk}{\gamma_{x,\regtext{pk}}}
\newcommand{\gmv}{\gamma_{v}}
\newcommand{\gma}{\gamma_{a}}
\newcommand{\gmpk}{\gamma_{\regtext{pk}}}

\newcommand{\betav}{\beta_\regtext{v}}
\newcommand{\betaa}{\beta_\regtext{a}}
\newcommand{\betapk}{\beta_\regtext{pk}}
\newcommand{\betapkmin}{\beta_\regtext{pk}^\regtext{min}}
\newcommand{\betapkmax}{\beta_\regtext{pk}^\regtext{max}}

\newcommand{\xobsm}{x^-_\regtext{obs}}
\newcommand{\xobsp}{x^+_\regtext{obs}}

%% file: sections/01_introduction.tex
\section{Introduction}\label{sec:introduction}

Autonomous unmanned aerial robots, such as quadrotors, can replace humans for dangerous tasks such as infrastructure inspection and search-and-rescue, which require navigating cluttered environments.
These robots are maneuverable, but often expensive and delicate.
Therefore, verifying they can operate \defemph{safely} (meaning, without collision) is important to enable them to provide such services.
Such verification is difficult because state space models of aerial robots are typically nonlinear and have at least 12 states \cite{Lee2010_geom_ctrl_SE3}.
In addition, these robots typically perform receding-horizon planning, where they execute a short trajectory while planning the next one, because the robot's sensor information is limited at any time.
So, the robot must plan trajectories that are verified as dynamically feasible and safe in real time.
This paper plans verified trajectories for quadrotor by extending the existing Reachability-based Trajectory Design (RTD) method.

\begin{figure}[t]
    \centering
    \includegraphics[width=0.9\columnwidth]{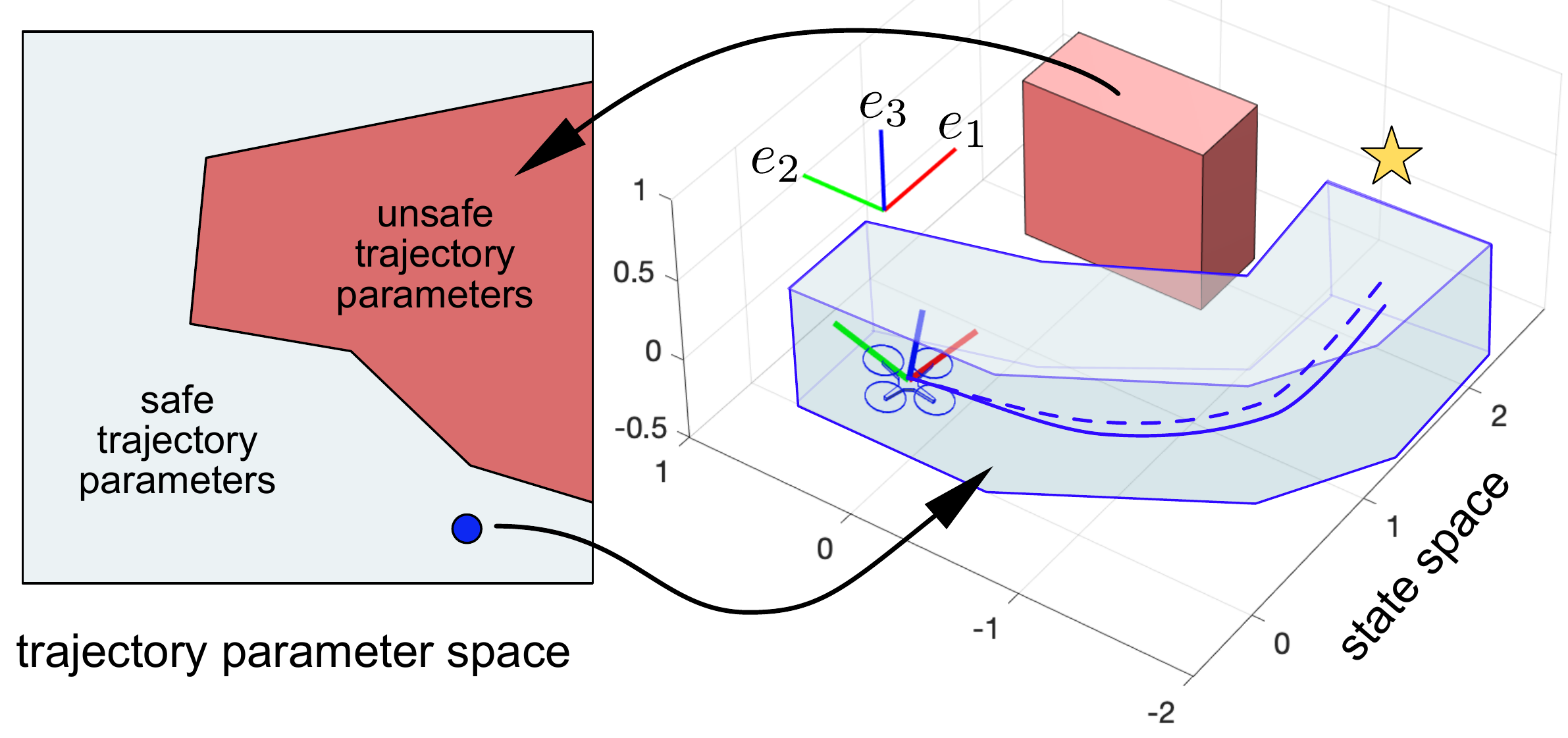}
    \caption{An overview of the proposed method.
    Points in the trajectory parameter space (on the left) correspond to desired trajectories (the dashed blue line on the right).
    The quadrotor, shown with its body-fixed coordinate frame, executes the solid blue trajectory, which has tracking error that depends on the desired trajectory.
    A Forward Reachable Set (FRS) is computed over all desired trajectories, plus tracking error and the body of the robot.
    The FRS is intersected with obstacles (red box on the right) to identify unsafe trajectories (red area on the left).
    Then, the subset of the FRS given by any safe trajectory parameter (blue tube on the right) will not intersect any obstalces.
    The particular desired trajectory shown attempts to reach a waypoint (gold star).
    A video is available at \href{http://roahmlab.com/quadrotor_rtd_demo}{\texttt{http://roahmlab.com/quadrotor\_rtd\_demo}}.
    }
    \label{fig:overview}
\end{figure}

\subsection{Related Work}
Quadrotor trajectory planners typically generate time-varying polynomial splines in position \cite{mellinger2011_min_snap,Richter2016_polynomial_traj_gen}.
Such splines have closed-form solutions for desired position, velocity, and higher derivatives, so they compute quickly \cite{mueller2015_quadrotor_splines}, and one can prove that they only lie within obstacle-free space \cite{Chen2016_online_safe_traj_gen}.
Since these splines are smooth, they can typically be tracked within 0.1 m of \defemph{tracking error} at speeds up to 8 m/s \cite{tal2018_INDI_control}; so, spline-based approaches typically treat tracking error implicitly, by dilating all obstacles by a fixed amount, which can be conservative.
These methods then rely on the quadrotor's trajectory-tracking \defemph{low-level controller} to ensure the quadrotor does not crash.
Since low-level controllers can compensate for aerodynamic and model disturbances \cite{tal2018_INDI_control,bouffard2012_lbmpc}, and large orientation deviations from a reference trajectory \cite{Lee2010_geom_ctrl_SE3,faessler2015_recovery_aggro_flight}, these approaches have been successful at navigating unknown, cluttered environments.

However, it is unclear how these methods can be extended to incorporate \defemph{trajectory-dependent uncertainty} (such as tracking error or aerodynamic disturbance) into online planning without dilating obstacles uniformly.
On the other hand, reachability-based methods address this issue by explicitly modeling trajectory-dependent uncertainty.
FasTrack, for example, computes tracking error as a function of control inputs and generates a feedback controller to compensate for it, which has been demonstrated on near-hover quadrotors \cite{herbert2017fastrack,bajcsy2018_multi_robot_collision_avoidance}.
Zonotope reachability can similarly compute tracking error, and has been shown on helicopters and cars, but requires computing a reachable set at every planning iteration, which can be too slow for real-time planning with high-dimensional system models \cite{althoff2015_helicopter,althoff2014online}.
Our prior work, called Reachability-based Trajectory Design (RTD), uses a parameterized continuum of low-dimensional trajectories, and computes a Forward Reachable Set (FRS) of the trajectories plus tracking error offline with Sums-of-Squares (SOS) programming.
Online, it maps obstacles to the trajectory parameter space via the FRS, then chooses a trajectory from the remaining safe set \cite{kousik2017safe}.
However, RTD has been only been shown for ground robots and low-dimensional models \cite{kousik2018_RTD_ijrr,vaskov2019_ACC,vaskov2019_RSS}.

\subsection{Contribution and Paper Organization}
This work extends RTD using zonotope reachability to produce guaranteed-safe reference trajectories for a quadrotor despite trajectory-dependent tracking error, as shown in Figure \ref{fig:overview}.
The contributions of this work are as follows.
First, we extend the dynamic models used in RTD from 2D to 3D (Section \ref{sec:models}).
Second, we present a method for approximating tracking error (Section \ref{sec:tracking_error}).
Third, we extend the FRS computation in RTD from SOS programming to zonotope reachability, which lets us increase the dimension of the FRS from 5 to 13 (Section \ref{sec:reachability}).
Fourth, we use the zonotope FRS to plan in real time, without recomputing the reachable set at every iteration, while adding trajectory-dependent tracking error online (Section \ref{sec:online_planning}).
We present simulation results in Section \ref{sec:results} and concluding remarks in Section \ref{sec:conclusion}.
A video is available at \href{http://roahmlab.com/quadrotor_rtd_demo}{\texttt{roahmlab.com/quadrotor\_rtd\_demo}}.

\subsection{Notation}
We adopt the following notation.
Variables, points, and functions are lowercase; sets and matrices are uppercase.
For a point $p$, $\{p\}$ denotes a set containing that point as its only element.
A multi-dimensional point or vector $v$ with elements $v_1, v_2$ is $(v_1,v_2)$.
Column vectors are in brackets in equations.
Subscripts indicate a subspace or description; superscripts in parentheses indicate an index.
An $n\times n$ identity matrix is $I_{n\times n}$.
An $m\times p$ matrix of zeros is $0_{m\times p}$.
An $n\times n$ matrix with diagonal elements $d_1, d_2, \cdots d_n$ is denoted $\mathrm{diag}(d_1,d_2,\cdots,d_n)$.
The positive real line is $ \R_+ = [0,+\infty)$.
The power set of $A$ is $\P(A)$.
Set addition is $A + B = \{a + b~|~a \in A, b \in B\}$.
For a state $x$, its first (resp. second) time derivative is $\dot{x}$ (resp. $\ddot{x}$).
Euclidean space in $n$ dimensions is $\R^n$.
The $3$-dimensional special orthogonal group is $\SO(3)$, with Lie algebra $\mathfrak{so}(3)$.

\begin{defn}\label{def:box}
A \defemph{box} is a set $B(c,l,w,h) \subset \R^3$ defined by its  center $c \in \R^3$ and its length, width, and height $l,w,h, \in \R_+$, that can be written
\begin{align}
    B(c,l,w,h) = \{c\} + \left[-\frac{l}{2},\frac{l}{2}\right]\times\left[-\frac{w}{2},\frac{w}{2}\right]\times\left[-\frac{h}{2},\frac{h}{2}\right],\label{eq:box}
\end{align}
where $[\cdot,\cdot]$ is an interval on each axis of $\R^3$.
A \defemph{cube} is a box with all sides of equal length $l$, denoted $B(c,l)$.
\end{defn}

%% file: sections/02_dynamic_models.tex
\section{Dynamic Models}\label{sec:models}

This section first specifies timing requirements for planning, and defines a fail-safe maneuver.
Next, it introduces the high-fidelity and trajectory-producing models used to apply RTD to a quadrotor.
Finally, it describes the robot's body.

RTD performs receding-horizon planning.
In each planning iteration, RTD first uses a \defemph{high-fidelity model} to estimate the robot's future position while it executes the current plan.
Then, RTD attempts to generate a new, safe plan starting from the future position estimate.
Planning with the high-fidelity model in real time is typically prohibitively computationally intensive, so RTD instead uses a lower-dimensional \defemph{trajectory-producing model} to generate a \defemph{desired trajectory} (also called a reference trajectory or plan).

To guarantee real-time operation, the robot must enforce a timeout on its online planning.
If a new plan cannot be found within this timeout, then the robot executes its previous plan.
To guarantee safe operation, RTD also requires that each desired trajectory incorporate a \defemph{fail-safe maneuver} that brings the robot to a safe state.
Then, if the robot cannot plan a new trajectory while executing the previous one, it can execute the remainder of its previous plan to come to a safe state.
We formalize these requirements as follows.
\begin{defn}\label{def:tau_plan_and_failsafe}
At each planning iteration, the robot has a \defemph{planning time} of $t\plan > 0$ in which to find a new, safe plan.
If no such plan is found, the robot is required to execute a collision-free \defemph{fail-safe maneuver} that brings it to a stationary hover.
\end{defn}
\noindent This fail-safe maneuver is reasonable because the quadrotor can land vertically from a hover.

\subsection{High-fidelity Model}
We denote the \defemph{high-fidelity model} as $f\hi: \T\times S \times U \to \R^{n\hi}$.
The \defemph{planning time horizon} is $\T = [t_0,\tfin]$.
Without loss of generality (WLOG), since we use receding-horizon planning, we let $t_0 = 0$ at the beginning of each planning iteration, so each planned trajectory is of duration $\tfin$.
The state space is $S = X\times V\times \Omega\times\SO(3)$ with state $s = (x,v,\om,R)$, where $x \in X \subset \R^3$ is position in the inertial frame; $v \in V \subset \R^3$ is velocity; $\om \in \Om \subset \R^3$ is angular velocity; and $R \in \SO(3)$ is attitude.
The inertial frame $X$ is spanned by unit vectors denoted $e_1$, $e_2$, and $e_3$ with $e_3$ pointing ``up'' relative to the ground, so $Re_3$ is the net thrust direction of the quadrotor's body-fixed frame.
We write the dynamics as per \cite{Lee2010_geom_ctrl_SE3}:
\begin{align}\begin{split}\label{eq:high-fidelity_model}
    \dot{x} &= v \\
    \dot{v} &= \tau Re_3 - m\mathrm{g}e_3 \\
    \dot{\om} &= J\inv\left(\mu - \om\times J\om\right) \\
    \dot{R} &= R\hat{\om},
\end{split}\end{align}
where $\hat{\cdot}: \R^3 \to \mathfrak{so}(3)$ is the \defemph{hat map} that maps a 3D vector to a skew-symmetric matrix \cite{Lee2010_geom_ctrl_SE3}.
The constant $\mathrm{g} = 9.81$ m/s\ts{2} is acceleration due to gravity.
The quadrotor's mass is $m \in \R$, and its moment of inertia matrix is $J \in \R^{3\times 3}$.
We assume $J$ is diagonal and constant, and write $J = \regtext{diag}(j_1,j_2,j_3)$.
The control input is $u = (\tau,\mu) \in U \subset \R^4$, where $\tau \in \R$ is net thrust and $\mu \in \R^3$ is the body moment; these inputs are related to rotor speeds as:
\begin{align}
    \begin{bmatrix} \tau \\ \mu \end{bmatrix} = 
    \begin{bmatrix} k_\tau & k_\tau & k_\tau & k_\tau \\
    0 & k_\tau\ell & 0 & -k_\tau\ell \\
    -k_\tau\ell & 0 & k_\tau\ell & 0 \\
    k_\mu & -k_\mu & k_\mu & -k_\mu
    \end{bmatrix}
    \begin{bmatrix} \om_{\regtext{rot},1}^2 \\ \om_{\regtext{rot},2}^2 \\ \om_{\regtext{rot},3}^2 \\ \om_{\regtext{rot},4}^2 \end{bmatrix},\label{eq:rotor_speed_conversion}
\end{align}
where $k_\tau$ and $k_\mu$ are rotor parameters, $\ell$ is the length from quadrotor center of mass to each rotor center, and $\om_{\regtext{rot},i}$ is the speed of the $i$\ts{th} rotor \cite{Powers2013_quadrotor_aero,Lee2010_geom_ctrl_SE3}.

\begin{assum}\label{ass:high-fidelity_model}
We assume commanded inputs can be achieved instantaneously (i.e., the rotor dynamics are fast compared to \eqref{eq:high-fidelity_model}), but that rotor speed is bounded (i.e., the inputs can saturate) \cite{Lee2010_geom_ctrl_SE3,mellinger2011_min_snap,mueller2015_quadrotor_splines}.
We also assume that the quadrotor has a \defemph{maximum speed} $\vmax > 0$ in any direction.
\end{assum}

\noindent We pick $\vmax = 5$ m/s, since aerodynamic drag can be compensated by rotor thrust up to $6$ m/s \cite{tal2018_INDI_control,hoffman2011_quadrotor_testbed}.
Note we are not concerned with model mismatch between the high-fidelity model and a real quadrotor.
However, RTD has been shown to handle model mismatch \cite{kousik2018_RTD_ijrr}.
We implement \eqref{eq:high-fidelity_model} with the specifications of an AscTec Hummingbird \cite{asctec_hummingbird,dong2015_hummingbird_specs} (see Table \ref{tab:sys_and_des_traj_params}).

\subsection{Trajectory-Producing Model}
\label{sec:models:trajectory-producing}

We use a trajectory-producing model that generates desired position trajectories with polynomials in time, generated separately in each coordinate of $X$, based on \cite{mueller2015_quadrotor_splines}, but modified so each trajectory has two piecewise polynomial segments, to include the fail-safe maneuver as in Definition \ref{def:tau_plan_and_failsafe}.
We first present a 1D model, then extend it to 3D.
Model parameters are in Table \ref{tab:sys_and_des_traj_params}.

Consider a 1D, twice-differentiable, desired position trajectory $p\des: \T\to \R$, with dynamics $f\oneD: \T\times K\oneD\to \R$:
\begin{align}
    \dot{p}\des(t;\kp) = f\oneD(t,\kp) = \frac{c_1(t,\kp)}{6}t^3 + \frac{c_2(t,\kp)}{2}t^2 + \kp_at + \kp_v,\label{eq:traj-prod_model_1D}
\end{align}
where the notation $p\des(t;\kp)$ indicates the trajectory parameterized by $\kp$.
We call $\kp = (\kp_v,\kp_a,\kp\peak) \in K\oneD \subset \R^3$ a \defemph{trajectory parameter}.
In particular, $\kp_a = \ddot{p}\des(0)$ is the initial desired acceleration, $\kp_v = \dot{p}\des(0)$ is the initial desired speed, and $\kp\peak$ is a desired \defemph{peak speed} to be achieved at a time $t\peak \in [t\plan,\tfin]$.
The values of $c_1, c_2$ are given by \cite[(64)]{mueller2015_quadrotor_splines} as
\begin{align}
    \begin{bmatrix}c_1(t,\kp) \\ c_2(t,\kp) \end{bmatrix} &= \frac{1}{(c_3(t))^3}\begin{bmatrix} -12 & 6c_3(t) \\ 6c_3(t) & -2(c_3(t))^2 \end{bmatrix}\begin{bmatrix} \Dl_v(t,\kp) \\ \Dl_a(t,\kp) \end{bmatrix},\\
    c_3(t) &= \begin{cases}
        t\peak &t \in [0,t\peak) \\
        \tfin - t\peak &t \in [t\peak,\tfin],
    \end{cases} \\
    \Dl_v(t,\kp) &= \begin{cases}
        \kp\peak - \kp_v - \kp_at\peak &t \in [0,t\peak) \\
        -\kp\peak &t \in [t\peak,\tfin],
    \end{cases} \label{eq:Dl_v}\\
    \Dl_a(t,\kp) &= \begin{cases}
        -\kp_a &t \in [0,t\peak) \\
        0 &t \in [t\peak,\tfin].
    \end{cases}\label{eq:Dl_a}
\end{align}
These dynamics produce a desired position trajectory that begins at the speed $\kp_v$ with acceleration $\kp_a$ at $t = 0$.
The trajectory accelerates to a speed of $\kp\peak$ at $t = t\peak$, at which point the desired acceleration is 0; the trajectory then slows down to desired speed and acceleration of $0$ at $t = \tfin$ (this is the fail-safe maneuver).
Notice that $c_3$, $\Dl_v$, and $\Dl_a$ are piecewise constant in $t$, with a jump discontinuity at $t\peak$.
Therefore, $c_1$ and $c_2$ are piecewise constant in $t$, which makes \eqref{eq:traj-prod_model_1D} a piecewise polynomial in time.
By construction, \eqref{eq:traj-prod_model_1D} and its derivative (acceleration) are continuous functions of time.
Note, a desired position trajectory can be translated arbitrarily, so we assume WLOG $p\des(0) = 0$.
Then, any desired position trajectory given by \eqref{eq:traj-prod_model_1D} is uniquely determined by $\kp$ for all $t \in \T$.

Note, we specify that $\kp_v, \kp_a$, and $\kp\peak$ lie in compact intervals $[\kp_v^-,\kp_v^+]$, $[\kp_a^-,\kp_a^+]$, and $[\kp\peak^-,\kp\peak^+]$, so $K\oneD$ is the Cartesian product of these three intervals.
The lower and upper bounds are reported in Table \ref{tab:sys_and_des_traj_params}.

We now make a 3D \defemph{trajectory producing model} by using the dynamics \eqref{eq:traj-prod_model_1D} for each dimension, and creating a larger parameter space $K = K\oneD\times K\oneD\times K\oneD \subset \R^9$.
For a trajectory $x\des: \T \to X$, we denote the dynamics as $f: T\times K \to \R^3$, so
\begin{align}
    x\des(t;k) = x\des(0) + \int_{\T} f(t,k)dt,\quad
    f(t,k) = \begin{bmatrix}
        f\oneD(t,\kp_1) \\
        f\oneD(t,\kp_2) \\
        f\oneD(t,\kp_3)
    \end{bmatrix},\label{eq:traj-prod_model}
\end{align}
with trajectory parameter $k = (\kp_1,\kp_2,\kp_3) \in K$, where each $\kp_i = (\kp_{v,i},\kp_{a,i},\kp\peaki)$ is the peak speed, initial speed, and initial acceleration in dimension $i = 1,2,3$.
As in the 1D case, WLOG we let $x\des(0) = 0$.
For notational purposes, let $k\peak = (\kp_{\regtext{pk},1},\kp_{\regtext{pk},2},\kp_{\regtext{pk},3})$ and similarly for $k_v$ and $k_a$.
Then $k = (k_v,k_a,k\peak)$ by reordering, and we denote $K = K_v \times K_a \times K\peak$.

By construction, \eqref{eq:traj-prod_model} includes the fail-safe maneuver specified by Definition \ref{def:tau_plan_and_failsafe}, and specifies what a \defemph{plan} is: at each planning iteration, the robot attempts to pick a new $k \in K$ that specifies a new desired trajectory $x\des$ to begin at $t\plan$.
We bound which $k$ can be chosen at each planning iteration.
First, per Assumption \ref{ass:high-fidelity_model}, speed is bounded: $\norm{k\peak}_2 \leq \vmax$.
Second, since $k\peak$ is a desired velocity and $k_v$ is the initial velocity, the quantity $\frac{1}{t\peak}\norm{k\peak - k_v}_2$ determines an approximate desired acceleration, leading to the following definition.
\begin{defn}\label{def:max_accel}
The \defemph{maximum desired acceleration} is $\amax > 0$.
We enforce a constraint at runtime that $\frac{1}{t\peak}\norm{k\peak - k_v}_2 \leq \amax$.
\end{defn}
\noindent Note that acceleration due to gravity is not included in the trajectory-producing model.
However, gravity is accounted for by the low-level controller we specify in Section \ref{subsec:low-level_controller}.

\subsection{The Robot as a Rigid Body}
The high-fidelity model \eqref{eq:high-fidelity_model} and trajectory-producing model \eqref{eq:traj-prod_model} only express the dynamics of the robot's center of mass.
However, for obstacle avoidance, one must consider the robot's entire body \cite{althoff2014online,kousik2018_RTD_ijrr}.
We do so as follows.

\begin{assum}\label{ass:rigid_body_box}
The robot is a rigid body that lies within a cube $\BQR = B(x(t),w) \subset X$ (centered at the robot's center of mass (COM) at any time with side length $w$).
We assume the box does not rotate, so it is large enough to contain the robot's body at any orientation.
We call this box the \defemph{body} of the robot.
\end{assum}
\noindent Though this is a conservative assumption, we find in Section \ref{sec:reachability} that it simplifies the computation of a reachable set for the robot's entire body, because we can first compute a reachable set of the robot's COM, then dilate the reachable set by $\BQR$.
Our quadrotor has dimensions of $0.54\times0.54\times0.0855$ m\ts{3} \cite{asctec_hummingbird,dong2015_hummingbird_specs}, so $w = 0.54$ m.
Note that $\ell = w/2$ is the distance from the COM to the center of each rotor.

\begin{table}[t]
\scriptsize
\centering
\caption{Implementation Parameters}
\begin{tabular}{c|c||c|c||c|c}
\multicolumn{2}{c||}{\textbf{Robot} \cite{asctec_hummingbird,dong2015_hummingbird_specs}} & \multicolumn{2}{c||}{\textbf{Control} \cite{mellinger2011_min_snap}} & \multicolumn{2}{c}{\textbf{Desired Traj.} \cite{mueller2015_quadrotor_splines}} \\
\hline
\textbf{Param.} & \textbf{Value} & \textbf{Param.} & \textbf{Value} & \textbf{Param.} & \textbf{Value} \\
\hline
$m$ & 0.547 kg & $G_x$ & $2.00 I_{3\times 3}$ & $t\plan$ & 0.75 s \\
$j_1, j_2$ & 0.0033 kgm\ts{2} & $G_v$ & $0.50 I_{3\times 3}$ & $t\peak$ & 1 s \\
$j_3$ & 0.0058 kgm\ts{2} & $G_R$ & $1.00 I_{3\times 3}$ & $\tfin$ & 3 s \\
$k_\tau$ & 1.5E-7 $\frac{\regtext{N}}{\regtext{rpm}^2}$ & $G_\om$ & $0.03 I_{3\times 3}$ & $\kp_v^\pm$ & $\pm5$ m/s \\
$k_\mu$ & 3.75E-9 $\frac{\regtext{Nm}}{\regtext{rpm}^2}$ & $\vmax$ & 5 m/s & $\kp_a^\pm$ & $\pm10$ m/s\ts{2} \\
$\ell$ & 0.27 m & $\amax$ & 3 m/s\ts{2} & $\kp\peak^\pm$ & $\pm5$ m/s \\
$\om_{\regtext{rot}}$ & 1100--8600 rpm & $\sensorhorizon$ & 12 m & &
\end{tabular}\label{tab:sys_and_des_traj_params}
\vspace*{-0.5cm}
\end{table}

%% file: sections/03_tracking_error.tex
\section{Tracking Error}\label{sec:tracking_error}

This section describes the low-level controller used to track desired trajectories; then defines tracking error as a set-valued, trajectory-dependent \defemph{tracking error function} $g$; and finally describes how to construct an approximation of $g$.
The purpose of $g$ is to include tracking error explicitly in the quadrotor's FRS (computed in Section \ref{sec:reachability}) for online planning (Section \ref{sec:online_planning}).

\subsection{Low-Level Controller}\label{subsec:low-level_controller}
Given any $k \in K$, the quadrotor uses a feedback controller $u_k: \T\times S \to U$ to track the trajectory parameterized by $K$.
For short, we say that $u_k$ \defemph{tracks $k$}.
This feedback controller can take any form, such as PID, LQR, or MPC; in this work, we use the PD controller specified in \cite[Section IV]{mellinger2011_min_snap}.
Recall that the quadrotor has states $s = (x,v,\om,R)$.
Consider a twice-differentiable desired position trajectory $x\des: \T\to\R$ as in \eqref{eq:traj-prod_model}.
Using the notation in \cite{mellinger2011_min_snap}, we specify a desired yaw $\psi(t) = 0$.
Then, $u(t)$ is uniquely determined by the current state $s(t)$, and the desired trajectory $x\des(t)$ and its derivatives, by leveraging differential flatness of the model \eqref{eq:high-fidelity_model} \cite{mellinger2011_min_snap,tal2018_INDI_control}.
At any time $t$, the state error used for feedback is
\begin{align}\begin{split}\label{eq:state_error}
    e_x(t) &= x(t) - x\des(t) \\
    e_v(t) &= v(t) - \dot{x}\des(t) \\
    e_R(t) &= \frac{1}{2}\left(R\des(t)^{\top}R(t) - R(t)^{\top}R\des(t)\right)^{\vee} \\
    e_{\om}(t) &= \om(t) - \om\des(t),
\end{split}\end{align}
where $(\cdot)^\vee: \mathfrak{so}(3) \to \R^3$ is the \defemph{vee map} that maps a skew-symmetric matrix to a 3D vector \cite{Lee2010_geom_ctrl_SE3}.
The desired control input $u_k(t,s(t)) = (\tau(t),\mu(t))$ is given by
\begin{align}\begin{split}\label{eq:fdbk_ctrl_u_k}
    \tau(t) &= \norm{-G_xe_x(t) - G_ve_v(t) + m\mathrm{g}e_3 + m\ddot{x}\des(t)}_2 \\
    \mu(t) &= - G_{\om}e_{\om}(t) -G_Re_r(t)
\end{split}\end{align}
where $R\des$ is found as in \cite[Section IV]{mellinger2011_min_snap} and $\om\des$ is found as in \cite[Section III]{mellinger2011_min_snap}.
In simulation, $\tau$ and $\mu$ are converted to rotor speeds and saturated using \eqref{eq:rotor_speed_conversion}.
The feedback gains and rotor speed saturation parameters are reported in Table \ref{tab:sys_and_des_traj_params}.

Note that, by including feedforward terms for angular acceleration and fulfilling other mild assumptions, one can modify \eqref{eq:fdbk_ctrl_u_k} to provably asymptotically drive tracking error to zero as time tends to infinity for any particular reference trajectory \cite{Lee2010_geom_ctrl_SE3}; however, since we are planning in a receding-horizon way, we find that \eqref{eq:fdbk_ctrl_u_k} tracks trajectories well over the time horizon $\T$ when commanding speeds up to $\vmax = 5$ m/s and $|\kp\peak - \kp_v| \leq 3$ m/s as in \eqref{eq:Dl_v}.
We express this notion of ``tracking well'' mathematically in the following subsections.
Loosely speaking, it means that $\norm{e_x(t)}_2 \leq 0.1$ m at any $t$.

\subsection{The Tracking Error Function}\label{subsec:defining_g}

Using the controller in \eqref{eq:fdbk_ctrl_u_k}, the quadrotor described by \eqref{eq:high-fidelity_model} cannot perfectly track trajectories produced by \eqref{eq:traj-prod_model}.
We call the position error term $e_x$ from \eqref{eq:state_error} the \defemph{tracking error}.
As shown in the literature, RTD can bound tracking error and incorporate it into a robot's FRS, which can then be used to plan safe trajectories \cite{kousik2018_RTD_ijrr,vaskov2019_ACC}.
Doing so requires the following assumption.

\begin{assum}\label{ass:compact_sets_lipschitz_dyn}
The sets $\T, S$, and $K$ are compact.
The high-fidelity model \eqref{eq:high-fidelity_model} is Lipschitz continuous in $t, s$, and $u$.
\end{assum}
\noindent Also notice that the desired position trajectory produced by \eqref{eq:traj-prod_model} is Lipschitz continuous in $t$ and $k$ because it is a piecewise polynomial on a compact domain.
Let $\idx: S \to X$ project points from $S$ to $X$ via the identity relation.
Now, we treat the tracking error as follows.

\begin{assum}\label{ass:tracking_error}
Suppose $s_0 \in S$ is an initial condition for \eqref{eq:high-fidelity_model} such that $\idx(s_0) = 0$.
Let $s: \T\to S$ be a trajectory of \eqref{eq:high-fidelity_model} beginning from $s_0$.
Let $t\in\T$, and let $k \in K$ be arbitrary but obeying Definition \ref{def:max_accel}.
Let $x\des: \T\to X$ be a trajectory of \eqref{eq:traj-prod_model} and recall $x\des(0) = 0$ WLOG.
We assume there exists a set-valued \defemph{tracking error function} $g: \T\times K \to \P(\R^3)$ for which
\begin{align}
    \idx(s(t; s_0,k)) \in \{x\des(t; k)\} + g(t,k).\label{eq:tracking_error_g}
\end{align}
We assume every $g(t,k)$ is compact.
\end{assum}
\noindent Note that $g(t,k)$ being compact is reasonable since $K$ is compact and the dynamics \eqref{eq:high-fidelity_model} are continuous, so the quadrotor cannot diverge infinitely far from any desired trajectory.
By Assumption \ref{ass:tracking_error}, for any desired trajectory, we can dilate the trajectory with $g$ to check if the high-fidelity model can collide with obstacles in $X$.
However, we instead combine $g$ with the FRS computed in Section \ref{sec:reachability}.
Then, we use the FRS to map obstacles to trajectories that are unsafe for the high fidelity model; so, instead of checking for collisions ``forward'' from trajectory space to $X$, we map obstacles ``backwards'' to eliminate unsafe trajectories.
Next, we implement an approximation of $g$.

\subsection{Implementation}\label{subsec:finding_g}

Computing $g$ as in \eqref{eq:tracking_error_g} is difficult in general for nonlinear systems such as \eqref{eq:high-fidelity_model} with more than 6 dimensions \cite{kousik2018_RTD_ijrr,herbert2017fastrack}.
Instead, we approximate $g$ by computing a function $\gapp: T \times K \to \P(\R^3)$ via sampling.
To justify our approach, we first compare the high-fidelity model \eqref{eq:high-fidelity_model} to a linear system (such comparisons are common for near-hover quadrotors \cite[Section 5.1]{hoffman2011_quadrotor_testbed}), plus other simplifications to make sampling tractable.
Then, we compute $\gapp$ with Algorithm \ref{alg:compute_gapp}.

\subsubsection{Simplifications to Enable Approximation}
A commonly-used approximation in quadrotor literature is that the rotational dynamics can be controlled on a faster time scale than the translational dynamics \cite{Michael2010_grasp_quadrotor,tal2018_INDI_control}.
So, to understand tracking error, we first suppose that the quadrotor's attitude $R$ is fixed.
Then, the dynamics become a linear system in each translational dimension, indexed by $i$, with control input $\tau$:
\begin{align}
    \begin{bmatrix} \dot{x}_i \\ \dot{v}_i \end{bmatrix} = \begin{bmatrix} v_i \\ (\tau (Re_3))\cdot e_i, \end{bmatrix}.
\end{align}
Recall that the quadrotor has a maximum speed of $\vmax$, so all possible initial velocities $k_v$ lie in the box $[-\vmax,\vmax]^3$.
Therefore, the initial speed in any direction is in the compact interval $[-\vmax,\vmax]$.
The position tracking error between a desired trajectory and a double integrator's executed trajectory, under linear feedback, is maximized when initial velocity is at the boundary of a compact interval:

\begin{prop}\label{prop:tracking_error_max}
Let $p \in \R$ be a state describing a 1D position, with dynamics $\ddot{p} = u$ and input $u \in \R$.
Let $p\des: \T \to \R$ be a twice-differentiable desired trajectory.
Suppose that $p(0) = p\des(0) = 0$, and $\dot{p}\des(0) \in \R$.
Suppose that $\dot{p}(0) \in [\dot{p}_0^-, \dot{p}_0^+] \subset \R$.
Let the input be given by linear feedback as
\begin{align}
    u(t) = \ddot{p}\des(t) + \kp_{\regtext{p}}(p(t) - p\des(t)) + \kp_{\regtext{d}}(\dot{p}(t) - \dot{p}\des(t)),
\end{align}
where $\kp_{\regtext{p}}, \kp_{\regtext{d}} \in \R$ are gains that can be chosen freely.
Then, for any $t \in \T$, the quantity $|p(t) - p\des(t)|$ is maximized when $\dot{p}(0) = \dot{p}_0^-$ or $\dot{p}(0) = \dot{p}_0^+$.
\end{prop}
\begin{proof}
Denote $z = (z_1,z_2)$ as the error system with
\begin{align}
    \dot{z}(t) = \begin{bmatrix}
    z_1(t) \\ z_2(t)
    \end{bmatrix} = \begin{bmatrix}
    p(t) - p\des(t) \\ \dot{p}(t) - \dot{p}\des(t)
    \end{bmatrix},\quad \dot{z}(0) = \begin{bmatrix}
    0 \\ \dot{p}(0) - \dot{p}\des(0)
    \end{bmatrix}.
\end{align}
We can rewrite this as
\begin{align}
    \dot{z} = \begin{bmatrix} 0 & 1 \\ \kp_{\regtext{p}} & \kp_{\regtext{d}}\end{bmatrix}z = Az,
\end{align}
which is an autonomous linear system with the solution
\begin{align}
    z(t) = e^{At}z(0) = \begin{bmatrix} a_{11} & a_{12} \\ a_{21} & a_{22}\end{bmatrix}z(0),
\end{align}
Pick $\kp_{\regtext{p}}, \kp_{\regtext{d}}$ so that $a_{12} \neq 0$.
Then
\begin{align}
    |z_1(t)| = |p(t) - p\des(t)| = |a_{12}(\dot{p}(0) - \dot{p}\des(0))|
\end{align}
is maximized when $\dot{p}(0) = \dot{p}_0^-$ or $\dot{p}(0) = \dot{p}_0^+$.
\end{proof}

Now, we build off of Proposition \ref{prop:tracking_error_max} to justify sampling, by discussing the relative influence of different states and parameters on tracking error.
In \eqref{eq:state_error}, $e_x$ is determined by the initial condition $s_0 \in S$ of the quadrotor and the choice of parameter $k \in K$.
Notice that not all $k$ are feasible (i.e., can be chosen as a new plan) for a given $s_0$; $k_v$ is determined by the initial velocity, and $k_a$ is determined by the current thrust $\tau$, attitude $R_0$, and gravity.
Similarly, given any $k$, only a subset of possible $s_0$ result in $k$ being feasible to the constraints of $\vmax$ and $\amax$.
Therefore, informally, we can think of $g(t,k)$ as the output of the following program:
\begin{align}
    g(t,k) = \underset{A}{\regtext{argsup}}\hspace{0.2cm}&\vol(A)\\
        \regtext{s.t}\hspace{0.5cm}&A = \{e_x(t;k)~|~s_0~\regtext{feas. to}~k\}.
\end{align}
This formulation for $g$ suggests that, to approximate $g$, we must compute $e_x(t;k)$ for every feasible combination of $(t,k,s_0)$ in $T\times K \times S$, which is a 22-dimensional space (if $R \in \SO(3)$ is represented with three Euler angles).
We reduce the size of the space to search with the following simplifications.
First, recall that the rotational dynamics operate on a much faster time scale than the translational dynamics.
Second, by Assumption \ref{ass:high-fidelity_model}, the net thrust $\tau$ is achieved instantaneously.
Therefore, any desired acceleration is achieved on a faster timescale than a desired velocity; in other words, tracking error is primarily caused by the velocity parameters $k_v$ and $k\peak$.
Furthermore, for any $k_v$, all feasible $k\peak$ are determined by $k_v$ and $\amax$ as in Definition \ref{def:max_accel}.
Therefore, we approximate $g$ by fixing $k_a$ and evaluating $e_x$ on points in the \defemph{tracking error subdomain} $D = T\times K_v$ of $g$.

Though $D$ is 4-dimensional, we want to compute $\gapp$ using the 12D nonlinear quadrotor model.
We begin by covering $D$ with a collection of subsets, then sampling to find maximum tracking error on each subset, illustrated in Figure \ref{fig:tracking_error_sampling}.
We denote the cover $\Dcover = \{D\upj\}_{j=1}^{n_{\Dcover}}$, where $n_{\Dcover} \in \N$.
To find $\gapp$, we first define maximum tracking error on elements of $\Dcover$ as follows.
\begin{defn}\label{def:max_trk_err_on_D_j}
Let $D\upj \in \Dcover$, and recall $D\upj \subset T\times K_v$.
Suppose $(t,k_v) \in D\upj$ and $k = (k_v,k_a,k\peak) \in K$ with $k\peak$ feasible to $k_v$.
Suppose $s_0$ is any initial condition of the quadrotor $s_0 \in S$ that is feasible given $k$.
The \defemph{maximum tracking error on $D\upj$} is a set $E\upj \subset \R^3$ such that $e_x(t;k) \in E\upj$.
\end{defn}
\noindent Now, we construct each element $D\upj \in \Dcover$ so that we can approximate $E\upj$ while only measuring $e_x$ at a finite number of points in $D\upj$.
Recall that, by Proposition \ref{prop:tracking_error_max}, for a given desired trajectory, the tracking error at any time is maximized at the ends of an interval of possible initial speeds, since the feedback law $u_k$ given by \eqref{eq:fdbk_ctrl_u_k} is as in the proposition when $R$ is fixed.
Therefore, we construct each $D\upj \subset \T\times K_v$ using intervals:
\begin{align}\label{eq:D_j}
    D\upj &= T\upj\times K_v\upj = \left[t_-\upj, t_+\upj\right] \times B\left(v\upj, \dl_v\upj\right),
\end{align}
where $v\upj \in K_v$ is the center of the cube $K_v\upj = B\left(v\upj, \dl_v\upj\right)$ of side length $\dl_v\upj$ as in Definition \ref{def:box}.

\subsubsection{Approximating the Tracking Error}
We compute $\gapp$ with Algorithm \ref{alg:compute_gapp}.
Consider a single $D\upj \in \Dcover$ (we iterate through the $D\upj$ starting on Line \ref{lin:for_D_j}).
Since $K_v\upj \subset K_v$ is a cube, it has eight vertices, in the set $N_v\upj \subset K_v\upj$, that are returned by \texttt{GetVertices} (Line \ref{lin:get_vertices}).
We iterate through each vertex (Line \ref{lin:for_vertex}) and compute tracking error.
To understand why, recall Proposition \ref{prop:tracking_error_max}; given any $t \in \T\upj$, the tracking error over all of $D\upj$ is largest at the vertices of $K_v\upj$.
We do not know a priori what $t \in \T\upj$ is likely to maximize the tracking error.
However, given any $k \in K$, we can forward-integrate the high-fidelity model to approximate $e_x(t;k)$ on $\T$, then take the maximum of $e_x(t;k)$ over $t \in T\upj$, since it is a compact interval.
So, given any $D\upj$, we only need to pick which $k\upn \in K$ to use, where $n = 1,\cdots,n\upj$ indexes the samples and $n\upj \in \N$.

To pick the samples $k\upn = (k_v\upn,k_a\upn,k\peak\upn)$, first recall the simplification where any desired acceleration can be achieved on a much faster timescale than any desired velocity or position.
Therefore, we specify $k_a\upn = 0$.
Also, recall that we are trying to maximize tracking error, and $k_a$ is set to the quadrotor's estimated acceleration at the beginning of each planning iteration (see Section \ref{sec:online_planning} Algorithm \ref{alg:online_planning}).
Therefore, for any $k_v$ and $k\peak$, the tracking error is likely to be larger for a trajectory with $k_a = 0$ as opposed to one with $k_a$ as the estimated acceleration; i.e., this simplification is conservative.

Now recall that $k_v\upn$ is specified by each vertex of $D\upj\setminus \T\upj$.
So, it is left to choose $k\peak\upn$.
To pick the $k\peak\upn$ samples, notice that the simplifying assumptions attempt to decouple rotation and acceleration from translation.
Up to this point, these assumptions have been useful for reducing the problem size.
But, we want to pick $k\peak\upn$ to maximize tracking error; therefore, for any $k_v\upn$, we want to choose $k\peak\upn$ to encourage the coupling of rotation, acceleration, and translation.
Since the cube $K_v\upj$ is aligned with the inertial frame axes, at any vertex $k_v\upn \in N_v\upj$, we can encourage coupling by choosing $k\peak\upn$ in the eight directions $(\pm 1,\pm 1, \pm 1)$.
In other words, for each vertex of $K_v\upj$, we create a box $K\peakj$ and sample its vertices:
\begin{align}
    k\peak\upn = b\upn\begin{bmatrix} \pm 1 \\ \pm 1 \\ \pm 1\end{bmatrix} + k_v\upn\label{eq:sample_kpeak}
\end{align}
where $b\upn \in \R$ is chosen as large as possible for each $k\peak\upn$ to obey the constraints
\begin{align}
    \frac{1}{t\peak}\norm{k\peak\upn - k_v\upn}_2 \leq \amax\quad\regtext{and}\quad\norm{k\peak\upn}_2 \leq \vmax.\label{eq:vmax_and_amax_cons}
\end{align}
The $k\peak\upn$ in \eqref{eq:sample_kpeak} are returned by \texttt{GetFeasPeakVels} (Line \ref{lin:get_feas_kpeak}).

By the procedure above, we sample eight $k\peak\upn$ for each of the eight $k_v\upn$, expressed by the inner for-loop (Line \ref{lin:for_kpeak}).
This results in $n\upj = 64$ samples $k\upn = (k\peak\upn,k_v\upn,0)$ (recall that $k_a\upn = 0$) to evaluate for each $D\upj$.
Denote the set of samples $K_{\regtext{sample}}\upj = \{k\upn\}_{n = 1}^{64}$.
For each $k\upn \in K_{\regtext{sample}}\upj$, we numerically forward-integrate the high-fidelity model to approximate $e_x(t;k\upn)$ on $\T\upj$.
This returns $e_x$ at the discrete set of times $\{t\upm\}_{m=1}^{m_t} \subset \T\upj$ such that $[t^{(1)}, t^{(m_t)}] = \T\upj$, where $m_t \in \N$.
Then, we approximate $E\upj$ as
\begin{align}
    E\upj \approx \texttt{convhull}\left(\bigcup_{t\upm \in T\upj}\left\{e_x(t\upm,k\upn) \mid k\upn \in  K_{\regtext{sample}}\upj\right\}\right) \subset \R^3.\label{eq:convhull_approx_E_j}
\end{align}
This approximate $E\upj$ is found for each $D\upj$; we store the pairs $(D\upj,E\upj)$ as a lookup table to represent $\gapp$.
So, if $(t,k_v) \in D\upj$, then $\gapp: T\times K \to \P(\R^3)$ is given by
\begin{align}
    \gapp(t,k) = E\upj,\label{eq:gapp_defn}
\end{align}
where $k = (k_v,k_a,k\peak)$ for any $k\peak \in K\peak$ that is feasible to $k_v$ and $\amax$, and with $k_a \in K_a$.

\subsubsection{Implementation Details}
We implement Algorithm \ref{alg:compute_gapp} with each $T\upj$ of duration $0.02$ s and each $K_v\upj$ of side length $0.7$ m/s, so $|\Dcover| = 102,900$.
Running Algorithm \ref{alg:compute_gapp} takes 0.8 hrs on a 3.1 GHz laptop in MATLAB, and produces a 8.6 MB lookup table.
The quadrotor dynamics are simulated as explained in Section \ref{subsec:sim_implementation}.

Next, in Section \ref{sec:reachability}, we compute an FRS of the trajectory-producing model \eqref{eq:traj-prod_model}.
We combine the FRS with the tracking error given by $\gapp$ to do online planning in Section \ref{sec:online_planning}.

\begin{figure}
    \centering
    \includegraphics[width=0.9\columnwidth]{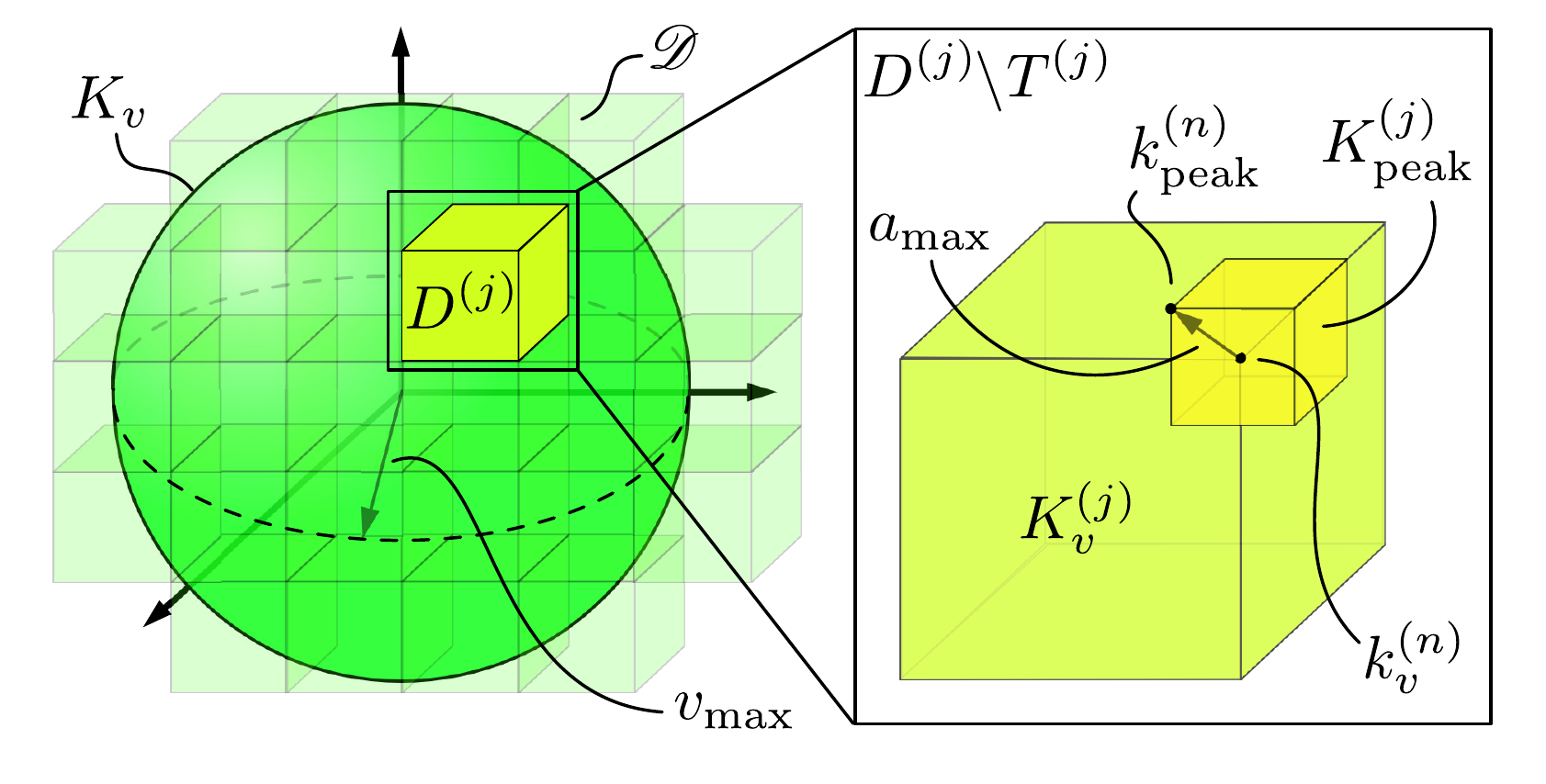}
    \caption{Sampling velocities in the space $K_v\times K\peak$ as in Algorithm \ref{alg:compute_gapp}.
    On the left, the ball $K_v \subset \R^3$ (in green) of initial velocities (radius $\vmax$) is covered by $\Dcover$ (light green) which contains boxes $D\upj \subset T\times K_v$; the time dimension is not shown here.
    Each $D\upj$ (in lime green) contains a box of initial velocities $K_v\upj$ (on the right), which has eight vertices, each of which is used as an initial speed sample.
    For each initial speed sample $k_v\upn$, eight possible peak speeds $k\peak\upn$ are sampled from a box that is a subset of $K\peak$ (the small yellow box) centered at $k_v\upn$; the samples $k\peak\upn$ are constrained as in \eqref{eq:vmax_and_amax_cons}.
    So, for each $D\upj$, there are 64 sampled initial/peak velocities; the tracking error is measured over all of the desired trajectories produced by these samples (with initial acceleration $k_a = 0$).}
    \label{fig:tracking_error_sampling}
\end{figure}

\begin{algorithm}[t]
\small
\begin{algorithmic}[1]
    \State {\bf Require:} $\Dcover$, $f\hi$ as in \eqref{eq:high-fidelity_model}, $f$ as in \eqref{eq:traj-prod_model}, $u_k$ as in \eqref{eq:fdbk_ctrl_u_k}
    
    \State {\bf Initialize:} $E\upj \leftarrow 0$ for $j = 1,\cdots,n_{\Dcover}$
    
%% D_j loop
    \State{\bf For:} $D\upj \in \Dcover$ as in \eqref{eq:D_j}\label{lin:for_D_j}

    \State\hspace{0.2in} $N_v\upj \leftarrow \texttt{GetVertices}\left(D\upj\setminus \T\upj\right)$ // vertices of $K_v\upj$\label{lin:get_vertices}
    
%% k_v vertices loop
    \State\hspace{0.2in}{\bf For:} $k_v \in N_v\upj$\label{lin:for_vertex}
    
    \State\hspace{0.4in} $N\peakj \leftarrow \texttt{GetFeasPeakVels}\left(k_v\right)$\label{lin:get_feas_kpeak}
    
%% k\peak loop
    \State\hspace{0.4in}{\bf For:} $k\peak \in N\peakj$\label{lin:for_kpeak}
    
    \State\hspace{0.6in}{\bf Reset state} $(x,v,\om,R) \leftarrow \left(0,k_v,0,I_{3\times 3}\right)$\label{lin:reset_quadrotor}
    
    \State\hspace{0.6in}{\bf Compute} $e_x\left(t; k\peak\right)$ as in \eqref{eq:state_error}\label{lin:compute_error}
    
    \State\hspace{0.6in}{\bf Update} $E\upj \leftarrow \texttt{convhull}\left(E\upj \cup \{e_x(t) \mid t \in \T\upj\}\right)$\label{lin:update_E_j}
    
    \State\hspace{0.4in}{\bf End}
    
    \State\hspace{0.4in}{\bf Store} $(D\upj, E\upj)$\label{lin:store_data}
    
    \State\hspace{0.2in}{\bf End}
    
    \State{\bf End}
\end{algorithmic}
\caption{\small Tracking Error Function Approximation (Offline)}
\label{alg:compute_gapp}
\end{algorithm}

%% file: sections/04_reachability.tex
\section{Reachability Analysis}\label{sec:reachability}

To produce safe trajectory plans, RTD first performs an offline FRS computation, then uses the FRS for online planning.
This section explains the offline FRS computation.
Given the time horizon $\T$, we define the \defemph{exact FRS} $\FRSexact \subset X \times K$ as all points in $X$ that are reachable by the trajectory-producing model \eqref{eq:traj-prod_model} plus tracking error:
\begin{align}\begin{split}\label{eq:FRS_defn}
    \FRSexact = \Big\{(x,k) \in X\times K\ \mid\ &\exists\ t\in\T\ \regtext{s.t.}\ x \in \{\tilde{x}(t)\}+g(t,k), \\
    &\dot{\tilde{x}}(t,k) = f(t,k),\,\regtext{and}\,\tilde{x}(0) = 0\Big\}.
\end{split}\end{align}
Prior work on RTD used SOS programming to compute the FRS \cite{kousik2017safe,kousik2018_RTD_ijrr,vaskov2019_ACC,vaskov2019_RSS}.
In this work, we instead use zonotope reachability via the CORA toolbox \cite{althoff2015_cora} for FRS computation, because we found it performs well for the 13D trajectory-producing model \eqref{eq:traj-prod_model}.
This section describes how to use CORA to find a \defemph{computed FRS} $\FRS$ that conservatively approximates the FRS of the trajectory-producing model \eqref{eq:traj-prod_model} \cite{althoff2008_conservative}:
\begin{align}
    \FRS \supseteq \Big\{(x,k) \in X\times K\ \mid\ &\exists\ t\in\T\ \regtext{s.t.}\ x(t) = \tilde{x}(t) \\
    &\regtext{and}\ \dot{\tilde{x}}(t,k) = f(t,k)\Big\}.\label{eq:FRS_conservative}
\end{align}
Then, per Assumption \ref{ass:tracking_error}, if the tracking error function $g$ is added to the computed FRS $\FRS$, one can conservatively over-approximate the exact FRS $F$ in \eqref{eq:FRS_defn}.
Conservatism of the computed FRS is necessary to ensure safety; if no trajectories of the computed FRS plus tracking error hit an obstacle, then no trajectories of the exact FRS can hit that obstacle.
Further, since the desired trajectories are only for the center of mass of the robot, we add the size of the robot's body $B\QR$ to $\FRS$ as well.
In this section, we compute $\FRS$.
We add $g$ and $B\QR$ in Section \ref{sec:online_planning}.

\subsection{Zonotopes}
CORA represents sets using zonotopes.
A zonotope $Z$ is a polytope in $\mathbb{R}^n$ that is closed under linear maps and Minkowski sums \cite{althoff2015_cora}, and is parameterized by its center $c \in \mathbb{R}^n$ and generators $g^{(1)}, ... g^{(p)} \in \mathbb{R}^n$.
A zonotope describes the set of points that can be written as the center $c$ plus a linear combination of the generators, where the coefficient $\beta^{(i)}$ on each generator must be between $-1$ and $1$:
\begin{align}
\label{eq:zono_long}
Z = \left\lbrace  y \in \mathbb{R}^n \ \Big| \ y = c + \sum_{i = 1}^{p} \beta^{(i)} g^{(i)},\ -1 \leq \beta\upi \leq 1 \right\rbrace 
\end{align}
For convenience, we concatenate the generators into an $n \times p$ \defemph{generator matrix} $G$, and the coefficients into a \defemph{coefficient vector} $\beta$:
\begin{align}
    G = \left[ g^{(1)}, g^{(2)}, ..., g^{(p)} \right]~\regtext{and}~\beta = [\bt^{(1)}, \bt^{(2)}, \cdots, \bt^{(p)}]^\top.
\end{align}
We can then rewrite \eqref{eq:zono_long} as
\begin{align}
    Z = \left\lbrace  y \in \mathbb{R}^n \ | \ y = c + G \beta,\ -1 \leq \beta \leq 1 \right\rbrace,
\end{align}
where $\geq$ and $\leq$ are applied elementwise.
From here on, for brevity we will assume that $\bt \in [-1,1]$, and will write the constraints explicitly when this is not true.

Boxes can be exactly represented as zonotopes:
\begin{align}
    B(c, l, w, h) = \left\lbrace  y \in \mathbb{R}^3\ \Big|\ \ y = c + \diag \bigg( \frac{l}{2}, \frac{w}{2}, \frac{h}{2} \bigg) \ \beta \right\rbrace
\end{align}
where $c$, $l$, $w$, and $h$ refer to the center, length, width, and height (Definition \ref{def:box}).
From here on, boxes and their zonotope representations are used interchangeably.
We define addition of a zonotope $Z \subset \R^3 \times \R^n$ with a box $B \subset \R^3$ as follows:
\begin{align}
    Z + B = Z + \left\lbrace y \in \R^3 \times \R^n \ \Big| \ y = \begin{bmatrix} y_B \\ 0_{1 \times n} \end{bmatrix}, y_B \in B \right\rbrace.
\end{align}
This is useful, e.g., when a zonotope is defined over the position and parameter spaces, but we want to add a box defined only in position space (note, this assumes WLOG that the first three rows of the zonotope $Z$'s center $c$ and generator matrix $G$ correspond to the quadrotor's position coordinates).

\subsection{Implementation}

Since the 3D trajectory-producing model uses the same 1D dynamics \eqref{eq:traj-prod_model_1D} separately in each dimension, we begin by computing the FRS for the 1D trajectory-producing model.
Then, we combine three 1D FRSes to create a single 3D FRS.

\subsubsection{1D Trajectory-Producing FRS}\label{subsec:1DFRS}
Recall that the 1D trajectory-producing model's position $p\des(t, \kp)$ depends only on time and the trajectory parameter $\kp = (\kp_v,\kp_a,\kp\peak) \in K\oneD \subset \R^3$.
Therefore, we want to compute the set of all positions that can be reached given a time interval $T$ and parameter set $K\oneD$.

CORA represents the FRS as a zonotope at each of a finite collection of compact \defemph{time steps} that are intervals in $\T$.
With a minor abuse of notation, we let $t$ act as an index, so that $Z\oneDt$ denotes the zonotope describing the 1D FRS over the time step containing $t$.
Note that each $Z\oneDt$ is a subset of the 1D position space $X_i$ (where $i = 1, 2, 3$) and parameters $K\oneD$:
\begin{align}
    Z\oneDt \subseteq X_i \times K\oneD.
\end{align}
CORA works by first linearizing the system dynamics at the beginning of each time step about the center of $Z\oneDt$, and obtaining the zonotope for the next time step by multiplying $Z\oneDt$ by an over-approximation of the matrix exponential over that time step.
CORA also accounts for linearization error; since the trajectory producing quadrotor model \eqref{eq:traj-prod_model_1D} does not depend on state, this error remains small in practice.

The 1D computed FRS $\FRS\oneD \subseteq X_i \times K\oneD$, $i = 1, 2, 3$, is the union of the zonotopes defined at each time step:
\begin{align}
    \FRS\oneD = \bigcup_{t \in T} Z\oneDt.
\end{align}

CORA requires specifying an initial set and dynamics to compute the FRS.
We treat $\kp \in K\oneD$ as states, and define the initial set $Z\oneDo$ as:
\begin{align}
\label{eq:zono_init}
Z\oneDo = \left\lbrace  y \in \mathbb{R}^4 \ | \ y = 0_{4 \times 1} + \diag \big( 0, \kp_v^+, \kp_a^+, \kp\peak^+ \big) \beta \right\rbrace,
\end{align}
where the first dimension is position, so initial position $p\des(0)$ is at the origin WLOG.
The dynamics given to CORA are $\dot{p}\des$ as in \eqref{eq:traj-prod_model_1D}, and $\dot{\kp}\peak, \dot{\kp}_v, \dot{\kp}_a = 0$ (since parameters are constant over a trajectory).

\subsubsection{3D Trajectory-Producing FRS}\label{subsec:3DFRS}
As described in \eqref{eq:traj-prod_model}, 3D trajectories of the trajectory-producing model can be constructed by concatenating 1D trajectories.
A similar process is followed to construct the 3D FRS $\FRS$ by concatenating 1D FRSes $\FRS\oneD$.
Specifically, given $Z\oneDt$ with center $c\oneDt$ and generator matrix $G\oneDt$, we construct $Z\indext$ as:
\begin{align}
    \label{eq:gen_matrix}
    Z\indext = \left\lbrace  y \in \mathbb{R}^{12} \ \Big| \ y = \begin{bmatrix}
    c\oneDt \\
    c\oneDt \\
    c\oneDt
    \end{bmatrix} + \begin{bmatrix}
    G\oneDt & 0_{4 \times p} & 0_{4 \times p} \\
    0_{4 \times p} & G\oneDt & 0_{4 \times p} \\
    0_{4 \times p} & 0_{4 \times p} & G\oneDt \\
    \end{bmatrix}\beta \right\rbrace, 
\end{align}
where $p$ is the number of generators in the matrix $G\oneDt$.
The 3D FRS $\FRS \subseteq X \times K_v \times K_a \times K\peak$ of the quadrotor's COM position and control parameters is then the union of $Z\indext$ through time.
\begin{align}
    \label{eq:FRSzono}
    \FRS = \bigcup_{t \in T}Z\indext.
\end{align}
\noindent Notice that $\FRS$ represents a continuum of initial conditions of \eqref{eq:traj-prod_model} since it is defined over $K_v$ and $K_a$.
Next, we discuss how to use $\FRS$ and the tracking error function $\gapp$ for online planning.

%% file: sections/05_online_planning.tex
\section{Online Planning}\label{sec:online_planning}

RTD plans trajectories online by intersecting the FRS $\FRS$ with obstacles to identify safe desired trajectories, then optimizes over this set to fulfill an arbitrary cost function (e.g., minimize distance to a waypoint, desired acceleration, power usage) \cite{kousik2017safe,kousik2018_RTD_ijrr,vaskov2019_ACC,vaskov2019_RSS}.
Past implementations of RTD used polynomial superlevel sets to represent the FRS, and were required to incorporate tracking error in the offline computation, so the intersection of the FRS with obstacles implicitly accounted for tracking error.
In contrast, this section details how to add tracking error to the FRS online.
We also discuss obstacle representation, and the optimization program solved at each receding horizon iteration.

\subsection{The Online Planning Algorithm}

RTD plans trajectories in a receding horizon way by running Algorithm \ref{alg:online_planning} at each planning iteration.
At the beginning of each planning iteration, at time $t = 0$ WLOG, the parameters $k_v$ and $k_a$ are set as the estimated velocity and acceleration of the high-fidelity model \eqref{eq:high-fidelity_model}:
\begin{align}
    \label{eq:k_assignment}
    k_v & = v(0), \quad k_a = \tau(0)R(0)e_3 - m\regtext{g}e_3.
\end{align}
These are found by forwarding-integrating the model given the previous plan for a duration of $t\plan$.

For this discussion, suppose we have $g$ as in Assumption \ref{ass:tracking_error}.
The tracking error associated with the initial condition is added to $\FRS$ to make it a conservative approximation of the exact FRS.
RTD attempts to find a safe trajectory by optimizing over the set of safe parameters.
These safe trajectory parameters are found by taking the complement of the intersection of the FRS $\FRS$ with the sensed obstacles.
A limited amount of time $t\plan$ is specified within which RTD attempts to choose the trajectory to be followed in the next iteration.
If no trajectory is found in time, the quadrotor continues executing its previous trajectory, which brings it to a stationary hover as per \eqref{eq:traj-prod_model_1D}.
A single iteration of the online planning algorithm is summarized in Algorithm \ref{alg:online_planning}.
Each step of the algorithm is explained in this section.

\begin{algorithm}[t]
\small
\begin{algorithmic}[1]
    \State {\bf Require:} $g$ as in Assum. \ref{ass:tracking_error}, $\FRS$ as in \eqref{eq:FRS_conservative}, $\sensorhorizon$ as in Assumption \ref{ass:sensor_horizon}, $s_0$ as in \eqref{eq:high-fidelity_model}, and cost function $J: K \to \R$, previous plan $x_{\regtext{prev}}: \T \to S$, $k_v, k_a$ as in \eqref{eq:k_assignment},  $A \leftarrow \emptyset,\ b \leftarrow \emptyset$
    
%% time loop loop
    \State $\obsset \leftarrow \texttt{SenseObstacles}\left(x_0, \sensorhorizon \right)$\label{lin:sense_obs}
    
    \State $\FRS_\epsilon \leftarrow \texttt{AddTrackingError}\left(\FRS, g, k_v \right)$ // error-augmented FRS\label{lin:add_tracking_error}
    
    % \State $\FRS\epsig \leftarrow \texttt{SliceFRS}\left(\FRS, k_v, k_a \right)$ // error-aug. FRS slice\label{lin:slice_FRS}

    \State{\bf For:} $Z\indext_\epsilon \in \FRS_\epsilon$ // for each zonotope in error-aug. FRS slice
    
%% obstacle loop
    \State\hspace{0.2in}{\bf For:} $O\upj \in \obsset$\label{lin:for_obstacle}
    
    \State\hspace{0.4in} $K\pkunsafe\uptj \leftarrow \texttt{IntersectObsWithFRS}\left(Z\indext\epsig, O\upj\right)$\label{lin:obs_intersect}
    
    \State\hspace{0.4in} $\left(A\uptj, b\uptj\right) \leftarrow \texttt{GenerateConstraints}\left(K\pkunsafe\uptj \right)$\label{lin:gen_constraint}
    
    \State\hspace{0.4in}{\bf Concatenate} $(A,b)$ $\leftarrow$ $[A ; A\uptj]$, $[b ; b\uptj]$ \label{lin:store_constraints}

    \State\hspace{0.2in}{\bf End}
    
    \State{\bf End}
    
    %% optimization
    \State $x\des \leftarrow \texttt{OptimizeTrajectory} \left(J, A, b, k_v, k_a,t\plan,x_{\regtext{prev}}\right)$\label{lin:trajopt}
    
    \State{\bf Return} $x\des$
    
    \end{algorithmic}
\caption{\small A Single Planning Iteration (Online)}
\label{alg:online_planning}
\end{algorithm}

\subsubsection{Obstacles and Sensing}
In this work, we represent obstacles as boxes in 3D:
\begin{defn}\label{def:obs}
    An obstacle $O \subset X$ is a box as in Definition \ref{def:box}, with center $c \in X$, and length, width, and height $l, w, h \in \R_+$.
    Obstacles are static with respect to time.
\end{defn}
\noindent This is not a restrictive definition, since obstacles are typically represented as occupancy grids composed of boxes \cite{Chen2016_online_safe_traj_gen}.
When moving through the world, we assume that a quadrotor has a limited range over which it can sense obstacles.
We refer to this as the quadrotor's \defemph{sensor horizon} $\sensorhorizon$.
\begin{assum}\label{ass:sensor_horizon}
At any time $t$, an obstacle is considered to be sensed if any point $x_\regtext{obs}$ of the obstacle $O$ is within the sensor horizon from the quadrotor's COM position $x$:
\begin{align}
    \norm{x(t) - x_\regtext{obs}}_2 \leq \sensorhorizon \quad \forall x_\regtext{obs} \in O.
\end{align}
\end{assum}
\noindent Note, to ensure safety, the sensor horizon $\sensorhorizon$ must be larger than the distance traveled by the longest desired trajectory plus $\vmax$ times $t\plan$ \cite[Theorem 35]{kousik2017safe}.

Let $O\upj \subset X$ denote the $j$\ts{th} sensed obstacle, whose position is given relative the quadrotor's current position, and $\nObs$ be the number of obstacles within the quadrotor's sensor horizon.
Finally, let $X_\regtext{obs} \subset X$ represent the union of all sensed obstacles:
\begin{align}
    \label{eq:X_obs}
    X_\regtext{obs} = \bigcup_{j \in \nObs} O\upj
\end{align}

\subsubsection{Tracking Error}
Tracking error must be included in the FRS to identify safe trajectories of the high fidelity model.
Recall $g$ is as in Assumption \ref{ass:tracking_error}.
For any $t, k$, we first overapproximate the tracking error $g(t,k)$ by a box.
Let $ \boxfun(\cdot): \P(\R^3) \to \P(\R^3) $ overapproximate a bounded set of 3D positions with a box.
We add a tracking error box to each zonotope $Z\indext$ comprising the FRS $\FRS$ to obtain the \defemph{error-augmented FRS} $\FRS_{ \epsilon}$.
We also add the box $B\QR$ representing the quadrotor's body:
\begin{align}
    Z\indext_{\epsilon} &= Z\indext + \boxfun(g(t, k)) + B\QR\label{eq:add_tracking_error_to_FRS}\\
    \FRS_{\epsilon} &= \bigcup_{t \in T} Z\indext_{\epsilon}.
\end{align}

\subsubsection{Unsafe Trajectories}
Recall that $k_v$ and $k_a$ are specified by the quadrotor's state at the beginning of each planning iteration, so online planning is performed over the peak speeds $K\peak$.
We intersect obstacles $O\upj$ with the error-augmented FRS $\FRS_\epsilon$ to identify the unsafe set $K\pkunsafe \subset K\peak$ that could cause a collision with an obstacle.
A peak speed $k\peak$ is \defemph{unsafe} if the position dimensions of $\FRS_\epsilon$ associated with $k\peak$ intersect an obstacle.
Here, we detail how obstacles are intersected with $\FRS_\epsilon$

Notice that $\FRS_\epsilon \subset X \times K_v \times K_a \times K\peak$ is defined over a continuum of positions and parameters.
Recall that $X_\regtext{obs}$ represents all sensed obstacle positions, and $k_v$ and $k_a$ are set as the initial velocity and acceleration of the quadrotor in the current planning step.
We obtain the \defemph{unsafe subset} $\FRS\unsafe$ by intersecting $\FRS_{\epsilon}$ with the obstacles and initial condition:
\begin{align}
    \label{eq:FRS_unsafe}
    \FRS\unsafe = \FRS_\epsilon \bigcap X_\regtext{obs} \times \{ k_v \} \times \{ k_a \} \times K\peak
\end{align}
The set of unsafe trajectory parameters $K\pkunsafe$ is the projection of $\FRS\unsafe$ onto the $K\peak$ subspace:
\begin{align}
    K\pkunsafe = \idpeak (\FRS\unsafe)
\end{align}
where $\idpeak: \P(X\times K) \to \P(K\peak)$ projects sets via the identity relation.

\subsubsection{Trajectory Optimization}
Let $J: K \to \R$ be an arbitrary cost function.
Then, online, we find
\begin{align}
    k\peak^* = {\regtext{argmin}}_{k\peak}\left\{J(k)\ \mid\ k\peak \notin K\pkunsafe\  \regtext{and feas. to \eqref{eq:vmax_and_amax_cons}}\right\},\label{prog:trajopt}
\end{align}
where $k = (k_v,k_a,k\peak)$.
If adding tracking error to the FRS, intersecting the FRS with obstacles, and running \eqref{prog:trajopt} complete within $t\plan$, then we return a desired trajectory $x\des$ as in \eqref{eq:traj-prod_model} parameterized by $(k_v,k_a,k\peak^*)$; otherwise, we continue executing the previously-found trajectory (which includes a fail-safe maneuver).

Now, we formalize that Algorithm \ref{alg:online_planning} is safe.
\begin{thm}\label{thm:online_planner_is_safe}
Suppose the quadrotor is described by \eqref{eq:high-fidelity_model} as in Assumption \ref{ass:high-fidelity_model}.
Suppose $g: T\times K \to \P(\R^3)$ is as in Assumption \ref{ass:tracking_error}.
Suppose the FRS $\FRS$ is found as in \eqref{eq:FRS_conservative}.
Suppose WLOG $t = 0$ and that the quadrotor is initially safe in a stationary hover.
Then, if the quadrotor plans in a receding-horizon way using Algorithm \ref{alg:online_planning}, it is safe for all time.
\end{thm}
\begin{proof}
This theorem follows from the conservative definitions of $\FRS$ and $g$, and from the fact that any planned trajectory contains a fail-safe maneuver.
In other words, by construction, the quadrotor always either executes a safe trajectory, or maintains a stationary hover.
\end{proof}
\noindent Theorem \ref{thm:online_planner_is_safe} is stated briefly to summarize how RTD either constructs safe plans or commands the robot to execute a fail-safe maneuver, by relying on conservatism.
For a more detailed treatment, see \cite[Remark 70]{kousik2018_RTD_ijrr}.

\subsection{Implementation}

We implement Algorithm \ref{alg:online_planning} as follows.

\subsubsection{Obstacles and Sensing}
Obstacles (as in Definition \ref{def:obs} and Assumption \ref{ass:sensor_horizon}) are given by \texttt{SenseObstacles} (Line \ref{lin:sense_obs}).
For our choice of $\vmax$, $t\peak$, and $\tfin$ (see Table \ref{tab:sys_and_des_traj_params}), we find that $d\sense = 12$ m is sufficient \cite[Theorem 35]{kousik2017safe}.

\subsubsection{Tracking Error in the FRS}
Recall that we can approximate $g$ with $\gapp$ as in \eqref{eq:gapp_defn}.
We implement \texttt{AddTrackingError} (Line \ref{lin:add_tracking_error}) using \eqref{eq:add_tracking_error_to_FRS} as written, but with $\gapp$ instead of $g$.

\subsubsection{Unsafe Trajectories}
The intersection in \eqref{eq:FRS_unsafe} to obtain $\FRS\unsafe$ is implemented as \texttt{IntersectObsWithFRS} (Line \ref{lin:obs_intersect}) by intersecting each zonotope $Z\indext_\epsilon$ comprising $\FRS_\epsilon$ with each obstacle $O\upj$.
We find that for the quadrotor model and obstacle representations we have chosen, the intersection in \eqref{eq:FRS_unsafe} can be computed exactly.
This requires the use of the following lemma regarding the structure of $Z\indext_\epsilon$.
\begin{lem}
    \label{lem:gen_matrix_decomp}
    The zonotope $Z\indext_\epsilon$ can be written as
    \begin{align}
        \label{eq:gen_matrix_smol}
        Z\indext_\epsilon = \left\lbrace  y \in \mathbb{R}^{12} \ \Big| \ y = \begin{bmatrix}
        c\indext_{\epsilon, 1} \\
        c\indext_{\epsilon, 2} \\
        c\indext_{\epsilon, 3}
        \end{bmatrix} 
        + \begin{bmatrix}
        G\indext_{\epsilon, 1} & 0_{4 \times 4} &  0_{4 \times 4} \\
        0_{4 \times 4} & G\indext_{\epsilon, 2} &  0_{4 \times 4} \\
        0_{4 \times 4} & 0_{4 \times 4} & G\indext_{\epsilon, 3} 
        \end{bmatrix} \beta \right\rbrace
    \end{align}
    where $c\indext_{\epsilon, i}$ and $G\indext_{\epsilon, i}$ take the form
    \begin{align}
        \label{eq:gen_matrix_ind}
        c\indext_{\epsilon,i} = \begin{bmatrix}
        c_x \\
        c_v \\
        c_a \\
        c\peak
        \end{bmatrix}, \qquad
        G\indext_{\epsilon,i} = \begin{bmatrix}
        \gmxv & \gmxa & \gmxpk & \epsilon\indext_i \\
        \gmv & 0 & 0 & 0\\
        0 & \gma & 0 & 0\\
        0 & 0 & \gmpk & 0
        \end{bmatrix}
    \end{align}
    where each $c, \gm \in \R$, $\gmv, \gma, \gmpk$ are nonzero, $\epsilon\indext_i \in \R_+$  and $i  = 1,2,3$.
\end{lem}
\begin{proof}
    Note, the rows of $c\indext_{\epsilon, i}$ and $G\indext_{\epsilon, i}$ represent the $x_i, \kp_{v,i}, \kp_{a,i}$, and $\kp\peaki$ dimensions respectively.
    We refer to columns of $G\indext_\epsilon$ with nonzero elements in the dimensions representing $\kp_{v,i}, \kp_{a,i}$, and $\kp\peaki$ as \defemph{$k$-dependent}, while the rest of the columns are \defemph{$k$-independent}.

    By construction (see \eqref{eq:gen_matrix}), $G\indext_\epsilon$ has the block diagonal structure in \eqref{eq:gen_matrix_smol}.
    It remains to be shown that the blocks can be written in the form shown in \eqref{eq:gen_matrix_ind}.
    By construction, the parameters $\kp_{v,i},\kp_{a,i}$, and $\kp\peaki$ have no dependence on each other, their dynamics are zero, and they are drawn from sets of nonzero measure.
    Each of $(\kp_{v,i},\kp_{a,i}, \kp\peaki)$ affects only the position $x_i$ in the corresponding dimension $i = 1, 2, 3$.
    This ensures columns with nonzero elements in rows corresponding to $\kp_{v,i},\kp_{a,i}$, or $\kp\peaki$ must have a single other nonzero element in the row corresponding to $x_i$.
    Finally, the $k$-independent columns containing $\epsilon\indext_i$ represent sources of error and uncertainty (specifically, tracking error, linearization error, and overapproximation of the body of the quadrotor) separately in each dimension (so, they are nonzero only in the row corresponding to $x_i$).
    Therefore, for each position dimension $x_i$, all columns representing the $k$-independent sources of error of the original generator matrix can be combined into a single column containing $\epsilon\indext_i$.
\end{proof}

Next, we describe how to find the unsafe set of control parameters $K\pkunsafe\uptj$ given a zonotope $Z\indext_\epsilon$, obstacle $O\upj$, and initial velocity and acceleration $k_v$ and $k_a$.
Let $\regtext{proj}_{X_i}: \P(x) \to \P(X_i)$ project sets from $X$ into the $i$\ts{th} position dimension via the identity relation.
\begin{thm}
    Given $O\upj$, let $[\xobsm, \xobsp] = \regtext{proj}_{X_i}(O\upj)$ be the interval that is the projection of the obstacle onto the $i$\ts{th} position dimension.
    Given time $t$ and obstacle $O\upj$, the set of unsafe control parameters in the $i$\ts{th} dimension $K\pkunsafei\uptj$ that could cause a collision with that obstacle at that time is given by
    \begin{align}
        \label{eq:betaminusbetaplus}
        K\pkunsafei\uptj &= \begin{cases}
            [ \gmpk \betapkmin , \gmpk\betapkmax ],  &\regtext{if} \ \betapk^- \leq 1 \ \regtext{and} \ \betapk^+ \geq -1\\
            \emptyset, &\regtext{otherwise,}
        \end{cases}
    \end{align}
    where
    \begin{align}
        \label{eq:betaminus}
        \beta\peak^- &= \frac{1}{\gmxpk} \left( \xobsm - \epsilon\indext_i - c_x - \gmxv \frac{\kp_{v,i} - c_v}{\gmv} - \gmxa \frac{\kp_{a,i} - c_a}{\gma} \right), \\
        \label{eq:betaplus}
        \beta\peak^+ &= \frac{1}{\gmxpk} \left( \xobsp + \epsilon\indext_i - c_x - \gmxv \frac{\kp_{v,i} - c_v}{\gmv} - \gmxa \frac{\kp_{a,i} - c_a}{\gma} \right),
    \end{align}
    and
    \begin{align}
        \betapkmin= \min(\betapk^-, -1),\quad \betapkmax = \max(\betapk^+, 1). 
    \end{align}
    The set of unsafe control parameters is then the box
    \begin{align}
        \label{eq:Kpkunsafe1D}
        K\pkunsafe\uptj = K\pkunsafeone\uptj \times K\pkunsafetwo\uptj \times K\pkunsafethree\uptj.
    \end{align}
\end{thm}
\begin{proof}
    First, notice that the effect of the position error term $\epsilon\indext_i$ in \eqref{eq:gen_matrix_ind} is equivalent to buffering the obstacle by $\pm\epsilon\indext_i$.
    Then, we want to solve for all coefficients $(\betav, \betaa, \betapk) \in \R^3$ such that:
    \begin{align}
    \label{eq:beta_intersection}
    c\indext_{\epsilon,i} + G\indext_{\epsilon,i}\begin{bmatrix}
    I_{3 \times 3} \\
    0
    \end{bmatrix}\begin{bmatrix}
    \betav \\ \betaa \\ \betapk
    \end{bmatrix}\in \begin{bmatrix}[\xobsm - \epsilon\indext_i, \xobsp + \epsilon\indext_i] \\
    \{ \kp_{v,i} \} \\
    \{ \kp_{a,i} \} \\
    K\peaki
    \end{bmatrix}
    \end{align}
    where the right hand side is 4-dimensional interval equal to $[\xobsm - \epsilon\indext_i, \xobsp + \epsilon\indext_i] \times \{ \kp_{v,i} \} \times \{ \kp_{a,i} \} \times K\peaki$.
    Notice that the initial velocity and acceleration constrain $\betav$ and $\betaa$, so that $\betav = (\kp_{v,i} - c_v)/\gmv$ and $\betaa = (\kp_{a,i} - c_a)/\gma$.
    Writing out the first row of \eqref{eq:beta_intersection}, we obtain
    \begin{align}
        c_x + \gmxv \betav + \gmxa \betaa + \gmxpk \betapk \in [\xobsm - \epsilon\indext_i, \xobsp + \epsilon\indext_i]
    \end{align}
    which give \eqref{eq:betaminus} and \eqref{eq:betaplus} (for more details on interval arithmetic, see \cite{althoff2015_cora}).
    Then, \eqref{eq:betaminusbetaplus} follows by examining the $4$\ts{th} row of the result of
    \begin{align}
        c\indext_{\epsilon,i} + G\indext_{\epsilon,i}\begin{bmatrix}
    I_{3 \times 3} \\
    0
    \end{bmatrix}\begin{bmatrix}
    \{ \betav \} \\ \{ \betaa \} \\ [\betapk^-, \betapk^+]
    \end{bmatrix}
    \end{align}
    and enforcing that $\betapk^-, \betapk^+ \in [-1, 1]$ if a collision is possible.
    Finally, because the intersection is computed separately in each position dimension, a parameter $k\peak \in K\peak$ is only unsafe if it is an element of the Cartesian product of 1D unsafe sets defined in \eqref{eq:Kpkunsafe1D}.
\end{proof}

\noindent The set of unsafe trajectory parameters $K\pkunsafe$ is then given by the union of each time and each obstacle's unsafe parameters:
\begin{align}
    K\pkunsafe = \bigcup_{t \in T, j \in \nObs}K\pkunsafe\uptj.
\end{align}

\subsubsection{Constraint Generation}
 To represent each unsafe set $K\pkunsafe\uptj$ as constraints for trajectory optimization, we use the function \texttt{GenerateConstraints} (Line \ref{lin:gen_constraint}).
Recall that each $K\pkunsafe\uptj \subseteq K\peak$ as in \eqref{eq:Kpkunsafe1D} is a 3-dimensional interval, which can therefore be represented as a box.
Let $c = (c_1, c_2, c_3) \in \R^3$ be the center of $K\pkunsafe\uptj$, and $l, w, h \in R$ be the length, width and height of $K\pkunsafe\uptj$.
We now discuss how to generate constraints to check whether $k\peak \in K\peak$ is contained in $K\pkunsafe\uptj$.
\begin{thm}
    Given $k\peak \in K\peak$, we can check if it is in $K\pkunsafe\uptj$ with:
\begin{align}
    \min(A\uptj k\peak + b\uptj) \ < \ 0 \implies k\peak \not\in K\pkunsafe\uptj\label{eq:cons_eval_test_pt}
\end{align}
where the $\min$ is taken over the elements of its argument, and $A\uptj$ and $b\uptj$ are:
\begin{align}
    \label{eq:constraints}
    A\uptj = \begin{bmatrix}
    1 & 0 & 0 \\
    -1 & 0 & 0 \\
    0 & 1 & 0 \\
    0 & -1 & 0 \\
    0 & 0 & 1 \\
    0 & 0 & -1
    \end{bmatrix}, \quad b\uptj = \begin{bmatrix}
    -c_1 + \frac{l}{2} \\
    c_1 + \frac{l}{2} \\
    -c_2 + \frac{w}{2} \\
    c_2 + \frac{w}{2} \\
    -c_3 + \frac{h}{2} \\
    c_3 + \frac{h}{2} \\
    \end{bmatrix}
\end{align}
\end{thm}
\begin{proof}
Notice that each row of $A\uptj$ and $b\uptj$ defines an affine operation that evaluates positive for any $k\peak \in K\pkunsafe\uptj$.
When at least one row of $A\uptj k\peak + b\uptj$ evaluates negative, $k\peak$ lies outside of the intersection of the positive half-spaces defined by $A\uptj$ and $b\uptj$.
Therefore, $k\peak$ lies outside of $K\pkunsafe\uptj$ when the min is negative.
\end{proof}

Each $A\uptj$ and $b\uptj$ are concatenated (\texttt{Concatenate}, Line \ref{lin:store_constraints}) into a single $A$ and $b$, so that the constraints for all obstacles can be efficiently checked with matrix operations.

\subsubsection{Trajectory Optimization}\label{subsec:trajopt}
The final step in online planning is trajectory optimization (Line \ref{lin:trajopt}).
This requires optimizing over $K\peak$, which is 3D, so \texttt{OptimizeTrajectory} uses brute force sampling.
It generates a ball of approximately 10,000 samples in $K\peak$, and evaluates the constraints \eqref{eq:cons_eval_test_pt} and \eqref{eq:vmax_and_amax_cons} on the samples (this takes 50--150 ms).
It eliminates all infeasible samples, then evaluates an arbitrary cost function $J: K\peak \to \R$ on the remaining samples.
The sampled point with the lowest cost is denoted $k\peak^*$, which defines a new safe trajectory $x\des$ given by \eqref{eq:traj-prod_model} with $k = (k_v, k_a, k\peak^*)$.
If no feasible $k\peak^*$ can be found within $t\plan$, the quadrotor continues executing the previous trajectory, which ends in a stationary hover.

%% file: sections/06_results.tex
\section{Results}\label{sec:results}

\begin{figure}[t]
    \centering
    \includegraphics[width=\columnwidth]{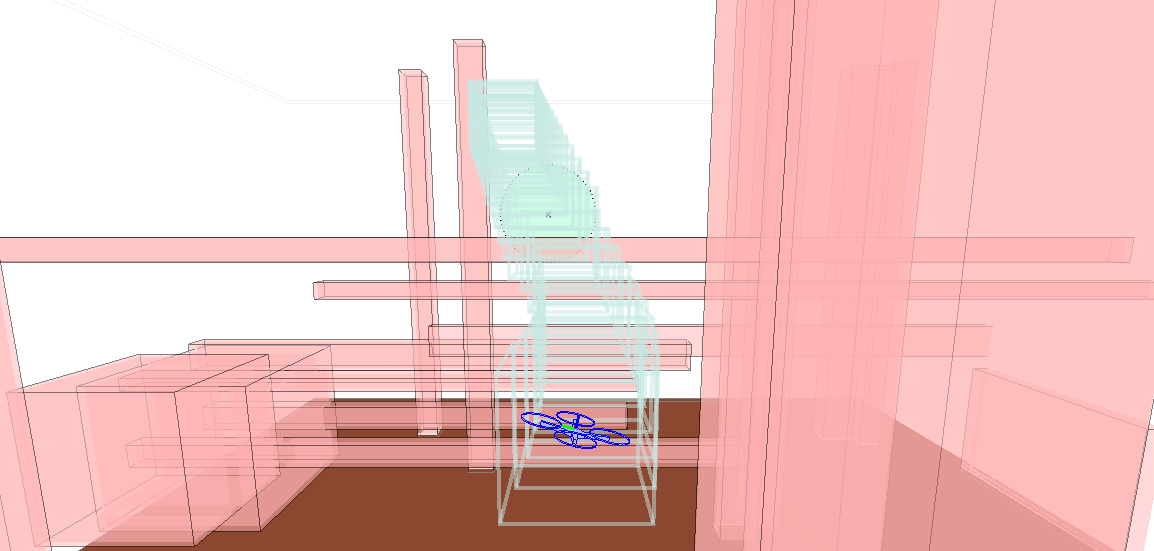}
    \caption{An example trajectory planned online in a cluttered environment with obstacles in light red and the ground in brown.
    The tube of light blue boxes, which does not intersect any obstacles, is the subset of the zonotope FRS for the current plan plus tracking error, so the quadrotor (in dark blue) is guaranteed to fly within the tube.
    The world and trajectory are shown in Figure \ref{fig:example_world}.}
    \label{fig:sim_example}
    \vspace*{-0.25cm}
\end{figure}

\begin{figure*}[t]
    \centering
    \includegraphics[width=0.9\textwidth]{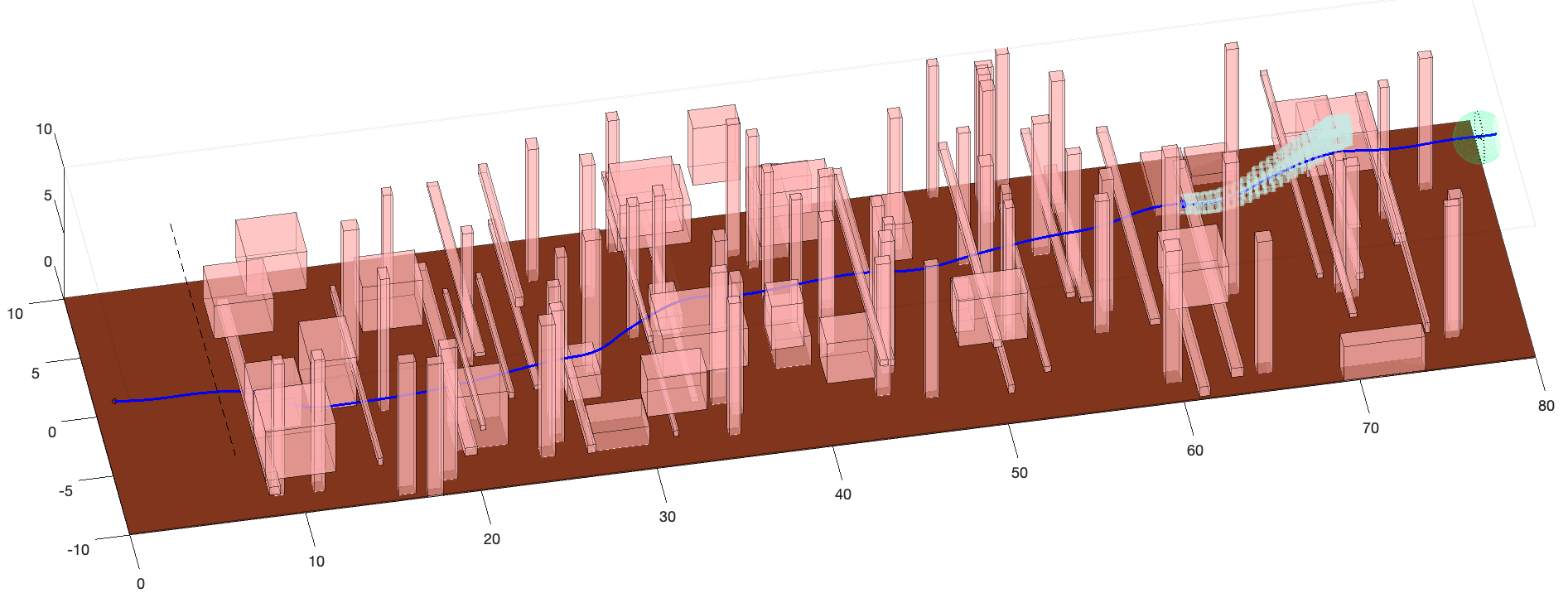}
    \caption{The example simulated world from Figure \ref{fig:sim_example}, with obstacles in light red, the ground in brown, world boundaries as axes, and the global goal as a light green sphere.
    A trajectory of the quadrotor is shown in dark blue, and goes from left to right.
    The quadrotor's reachable set (light blue) is shown for the same planning iteration as in Figure \ref{fig:sim_example}.}
    \label{fig:example_world}
\end{figure*}

All simulations, along with the tracking error and zonotope reachability computations, are performed in MATLAB.
We simulate an AscTec Hummingbird quadrotor \cite{asctec_hummingbird} at up to 5 m/s.
The system parameters are given in Table \ref{tab:sys_and_des_traj_params}.

\subsection{Simulation Implementation}\label{subsec:sim_implementation}
We simulate 500 cluttered worlds with 120 random obstacles each.
An example simulation is shown in Figures \ref{fig:sim_example} and \ref{fig:example_world} with a trajectory that RTD planned and executed in a receding-horizon way.
Each world is $80\times20\times10$ m in volume, with a random start location at one end and a random goal location at the other.
Note that the simulation environment performs collision checking of the body of the quadrotor with obstacles separately from how we generate and enforce constraints in Algorithm \ref{alg:online_planning}.

The quadrotor's dynamics \eqref{eq:high-fidelity_model} are simulated by Euler integration with a 5 ms time step; the rotation matrix dynamics are implemented as Lie-Euler integration on $\SO(3)$ as in \cite[(7)]{celledoni2014_lie_group_integrator} with $F_{y_n} = \hat{\om}_n$.
This was done to avoid Euler angle singularities.
Euler integration was found empirically to match a Runge-Kutta/Munthe-Kaas 4\ts{th} order method within millimeters in the quadrotor's position dimensions over the time horizon $\T$, while taking approximately 25\% of the computation time.
We include the numerical integration error as tracking error in the computation of $\gapp$ in Section \ref{subsec:finding_g}.

At each planning iteration, the quadrotor is given information about obstacles within a 12 m sensor horizon as in Assumption \ref{ass:sensor_horizon}, along with the world boundaries as obstacles.
The quadrotor is given $t\plan$ s to run Algorithm \ref{alg:online_planning} at each iteration; in other words, the quadrotor is required to plan in real time.

We ran two different implementations of Algorithm \ref{alg:online_planning}: one with a constant tracking error of 0.1 m, and one with trajectory-dependent tracking error computed as in Section \ref{subsec:finding_g}.
The distance $0.1$ m is the maximum tracking error found in any direction from computing the trajectory-dependent tracking error function $\gapp$.
The cost function used at each planning iteration (as in Section \ref{subsec:trajopt}) is to minimize the distance between the quadrotor and a waypoint at the time $t\peak$; the waypoint is placed 5 m ahead of the quadrotor along a straight line between the robot and the global goal.
Note that this choice of waypoint is deliberate to force the quadrotor into situations where it has to execute a fail-safe maneuver.

\subsection{Simulation Results}
The quadrotor never crashed.
With constant tracking error of 0.1 m, it reached the goal in 84.8\% of trials.
With trajectory-dependent tracking error, it reached the goal in 91.2 \% of trials.
Note, we did not expect 100\% of goals reached, since the trials used randomly-generated obstacles, so some simulated worlds have no feasible path between start and goal.
This result confirms that including trajectory-dependent tracking error reduces conservatism.

%% file: sections/07_conclusion.tex
\section{Conclusion}\label{sec:conclusion}

We propose Reachability-based Trajectory Design (RTD) as a method for enabling autonomous quadrotors to plan aggressive, safe trajectories in unforeseen, cluttered environments.
This work extends RTD to a 13D system with zonotope reachability; provides an approximation of trajectory-dependent tracking error for a high-dimensional model of a quadrotor; and introduces a novel method to use zonotopes for safe planning online.
The proposed method is demonstrated over 500 simulations in random cluttered environments at speeds up to 5 m/s, with zero crashes.
In future work, we will implement RTD on hardware, and explore more types of trajectory-dependent uncertainty.